\documentclass{article} 
\usepackage{scale,times}

\usepackage{url}            
\usepackage{booktabs}       
\usepackage{amsfonts}       
\usepackage{amsmath}
\usepackage{amssymb}
\usepackage{nicefrac}       
\usepackage{microtype}      
\usepackage[table]{xcolor}         
\usepackage{bm}
\let\openbox\relax
\usepackage{amsthm}
\usepackage{makecell}
\usepackage{graphicx}
\usepackage{multirow}
\usepackage{subcaption}
\usepackage{tcolorbox}
\usepackage{pifont}
\usepackage{wrapfig}
\usepackage{enumitem}
\usepackage{algorithm}
\usepackage{algpseudocode}
\usepackage{fontawesome5}
\usepackage{hyperref}

\runninghead{MirrorMoE: Rethinking Routing Exploration in Modern MoEs}   

\newtheorem{theorem}{Theorem}[section]
\newtheorem{proposition}[theorem]{Proposition}

\theoremstyle{definition}
\newtheorem{definition}[theorem]{Definition}

\theoremstyle{remark}

\newcommand{\G}{\mathbf{G}}

\newcommand{\h}{\mathbf{h}}

\newcommand{\topk}{\mathsf{Top}}

\definecolor{mydarkblue}{rgb}{0,0.15,0.55}

\title{Rethinking Routing Exploration \\in Modern Mixture-of-Experts}

\author{Bingcong Li\aff{1,2}
\and Hao Ma\aff{2}
\and Michael Muehlebach\aff{3,4}
\and Niao He\aff{1}
\and Yi Ma\aff{2}
}

\affiliations{%
  \affil{1}{ETH Zurich}%
  \affil{2}{The University of Hong Kong}%
  \affil{3}{Max Planck Institute for Intelligent Systems}%
  \affil{4}{WU Vienna}%
}
  
\emails{
\texttt{\{bingcong.li, niao.he\}@inf.ethz.ch},
 \{haoma, mayi\}@hku.hk, \\
michael.muehlebach@tuebingen.mpg.de
}

\begin{document}

\maketitle

\begin{abstract}

Mixture-of-experts (MoE) models have become an important approach for scaling model parameters.
Modern MoE architectures increasingly combine two design choices: fine-grained expert segmentation and shared experts. 
The former increases expressiveness by enabling more expert combinations at a fixed activation ratio, and the latter introduces experts that are active for every token.
We identify an \textit{exploration dilemma} between these designs:  shared experts can make routing less exploratory by reducing the variability of token representations. 
This conflicts with the greater need for exploration induced by fine-grained expert segmentation due to the combinatorially enlarged routing space. Consequently, the additional expressiveness offered by fine-grained segmentation may not be fully realized in practice. 
MirrorMoE is proposed to reconcile this exploration dilemma by shrinking the exploration space through token-adaptive suppression of unnecessary experts.
MirrorMoE consistently improves performance across model scales. Code is available \href{https://github.com/BingcongLi/MirrorMoE/tree/main}{here}.

\end{abstract}

\section{Introduction}

Mixture-of-experts (MoE) models have been instrumental in scaling large language models (LLMs) to trillions of parameters and beyond \citep{lepikhin2020gshard,fedus2022switch}. MoEs adopt a sparse conditional-computation paradigm, where each token activates only a subset of experts, or subnetworks. During pretraining, recent work has shown that MoEs can outperform their dense counterparts under matched total parameter counts, compute, and data budgets \citep{li2026mixture}. At inference, sparse activation allows computational cost to grow sublinearly with total parameters. Thanks to these advantages, MoEs have become a core architecture underlying model families such as DeepSeek \citep{dai2024deepseekmoe,liu2024deepseek}, Kimi \citep{kimik2,kimik3}, and Qwen \citep{qwen1.5,yang2025qwen3}, among many others \citep{jiang2024mixtral,abdin2024phi}.

At the core of MoEs is the MoE layer, which consists of a router and $E$ experts, each implemented as a feedforward network (FFN). 
For each input token, the router selects $K$ experts ($K < E$) for sparse activation \citep{lepikhin2020gshard,fedus2022switch}. Building on this basic design, recent MoEs have increasingly adopted two notable architectural advances.
The first is \textit{fine-grained expert segmentation} \citep{kimik3,liu2024deepseek}. Recent MoEs tend to use larger numbers of both total and activated experts, i.e., larger $E$ and $K$. For example, at the same activation ratio $K/E$, $(E,K)=(64,8)$ is often preferred over $(E,K)=(8,1)$. Since there are $\binom{E}{K}$ expert combinations, the latter yields richer routing space for more expressive token representations \citep{elango2026latentmoe}.
The second advance is \textit{shared experts}. Activated for every token, shared experts capture common knowledge that would otherwise be learned redundantly by routed experts. This allows routed ones to devote more capacity to specialization.
The combination of the above architectural choices have become standard in modern MoEs \citep{dai2024deepseekmoe,liu2024deepseek,qwen1.5}.

This work reexamines the interaction between these two trends from the angle of \textit{routing exploration}. 
The goal of exploration is to search for alternative expert combinations during pretraining to identify (near-)optimal routes for each token.
On the one hand, realizing the expressive potential of fine-grained expert segmentation requires exploring a \textit{combinatorially} larger routing space within each MoE layer.
For example, the number of expert subsets increases from $8$ when $(E,K)=(8,1)$ to ${64 \choose 8}\approx 4.4\times10^9$ with $(E,K)=(64,8)$, i.e., several orders of magnitude increase. 
On the other hand, our analysis reveals an asymmetry during pretraining: shared experts dominate the learned token representations, leaving routed experts with less influence on the layer output. Consequently, even activating different routed experts only leads to limited variation in the learned features, making alternative routes harder to distinguish and thereby undermining the goal of exploration.

This exploration dilemma has tangible consequences for shared-expert MoEs. To see this, we train matched MoE models with and without shared experts at 4 scales, ranging from 100M to 2.4B total parameters. At each scale, we fix the total number of experts $E$ and vary the number of (total) activated experts $K$ over $\{6,8,10\}$. 
A larger $K$ simultaneously increases model expressiveness and demands more strategic exploration due to enlarged routing space. 
As shown in Fig.\ref{fig.scaling}, shared experts no longer provide benefits over their non-shared counterparts as $K$ increases. 
More strikingly, the performance of shared-expert MoEs can even degrade as $K$ increases in Figs.\ref{fig.scaling-498m} and \ref{fig.scaling-2.4b}, 
suggesting that the inefficient exploration hinders realization of the additional expressiveness from larger $K$.

MirrorMoE is proposed to reconcile the conflicting architecture designs above. 
It adaptively suppresses routed experts on a token basis, reducing the exploration space to better match the model's exploration capability. 
Technically, MirrorMoE is derived by mirroring the routing logits via sign negation, hence its name. As shown in Fig.\ref{fig.scaling}, MirrorMoE improves the performance of shared-expert MoEs and attains expressiveness gains from larger $K$. In summary, our contributions are as follows:

\begin{enumerate}[leftmargin=0.4cm]
    \item[\ding{118}] We identify an \textbf{exploration dilemma} in modern MoEs: the enlarged routing space induced by fine-grained expert segmentation demands greater exploration, while shared experts make routing less explorative. As a result, the benefits of shared experts diminish as the routing space grows.
    
    \item[\ding{118}] We propose \textbf{MirrorMoE}, which adaptively shrinks the exploration space to reconcile competing demands above. MirrorMoE is conceptually simple, and its adaptivity is theoretically characterized.

    \item[\ding{118}] We demonstrate that MirrorMoE addresses the exploration dilemma across MoEs of different scales. On the 2.4B model, MirrorMoE improves pretraining and downstream performance by 0.67 and 1.64 points, respectively. Moreover, it enables shared-expert MoEs to consistently outperform their non-shared counterparts and restores the performance gains from increasing $K$; see Fig.\ref{fig.scaling}.
    
\end{enumerate}

\begin{figure}[t]
  \centering
  \hspace{0.1cm}
  \begin{subfigure}{0.24\textwidth}
    \centering
    \vspace{-0.1cm}
    \includegraphics[width=\textwidth]{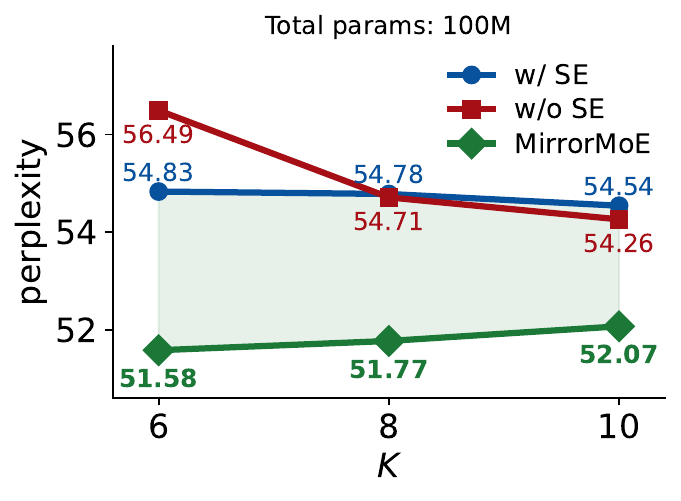}
    \vspace{-0.4cm}
    \caption{100M}
    \label{fig.scaling-100m}
  \end{subfigure}
  \begin{subfigure}{0.24\textwidth}
    \centering
    \vspace{-0.1cm}
    \includegraphics[width=\textwidth]{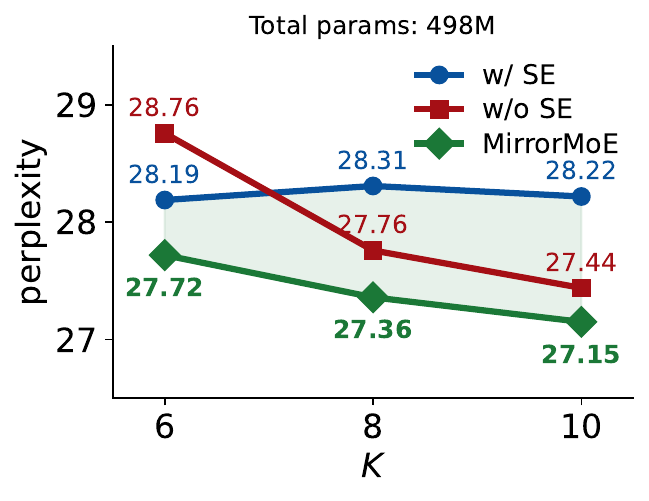}
    \vspace{-0.4cm}
    \caption{498M}
    \label{fig.scaling-498m}
  \end{subfigure}
\begin{subfigure}{0.24\textwidth}
    \centering
    \vspace{-0.1cm}
    \includegraphics[width=\textwidth]{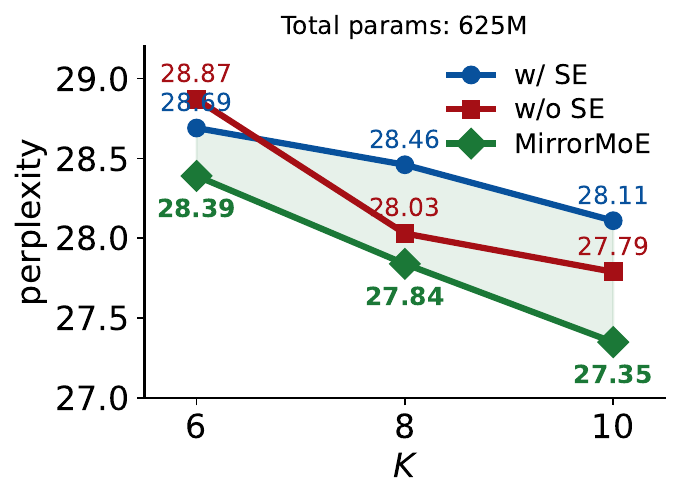}
    \vspace{-0.4cm}
    \caption{625M}
    \label{fig.scaling-625m}
  \end{subfigure}
\begin{subfigure}{0.24\textwidth}
    \centering
    \vspace{-0.1cm}
    \includegraphics[width=\textwidth]{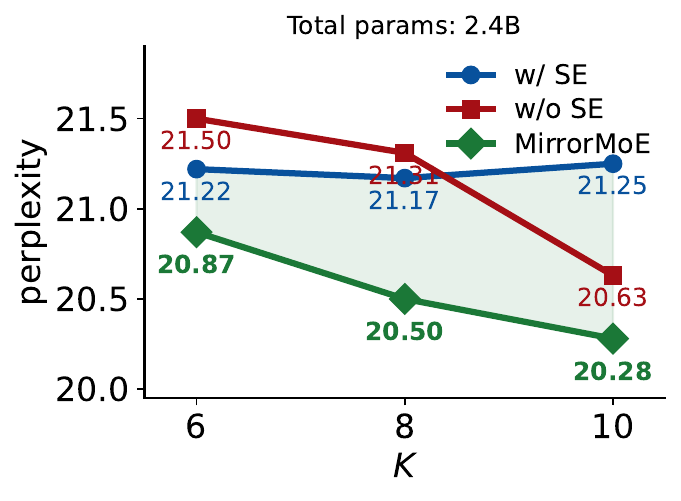}
	\vspace{-0.4cm}
    \caption{2.4B}
    \label{fig.scaling-2.4b}
  \end{subfigure}
  \vspace{-0.2cm}
  \caption{Comparison of MoEs with and without shared experts (SE) and the proposed MirrorMoE. MirrorMoE consistently i) achieves the best performance, ii) realizes the expressiveness gains from larger $K$, and iii) enables shared-expert MoEs to outperform their non-shared counterparts.}
  \label{fig.scaling}
  \vspace{-0.4cm}
\end{figure}

\subsection{Related work}\label{sec.related-work}

\textbf{MoEs.} Modern sparse MoEs originated with the sparsely gated MoE layer of \citet{shazeer2017outrageously} and were extended to large-scale Transformers in e.g., \citet{lepikhin2020gshard,fedus2022switch}. Later work focused on improving scalability, stability, and expert specialization. GLaM \citep{du2022glam} demonstrated that sparse language models could scale beyond one trillion parameters, while ST-MoE \citep{zoph2022st} improved training stability. MoEs have since become widely adopted. Beyond those mentioned earlier, other representative examples include Gemma4, OLMoE and gpt-oss \citep{team2026gemma, muennighoff2025olmoe, agarwal2025gpt}.
Across this development trajectory, two trends have clearly emerged: fine-grained expert segmentation and shared experts.

\textbf{Tension between expressiveness and pretraining.} A growing line of work shows that a model's expressive capacity may not be readily realized without carefully designed training procedures \citep{lion2026polar,zhang2026ancre,jha2026same,wei2025benefits}. 
Within the context of MoEs, related evidence is provided by \citet{jin2024moe++}, which augments learnable experts with fixed zero, copy, and constant experts. Although standard FFN experts can, in principle, learn these operations, explicitly incorporating them still improves empirical performance. 
The benefit of directly defining these operations, rather than learning them through training, suggests that the available expressiveness may not be fully realized in practice. 
Our work further shows that the additional expressiveness provided by fine-grained expert segmentation is not naturally attained in shared-expert MoEs.

Additional related work on MoE routing and systems is deferred to App.\ref{apdx.related-work} due to limited space.

\section{The Exploration Dilemma in Shared-Expert MoEs}

This section first reviews MoEs and then examines the effect of shared experts on exploration.

\subsection{MoE recap}

MoEs replace the FFN in each Transformer layer with a collection of FFNs, each referred to as an \emph{expert}. Among $E$ experts at layer $l$, there are two types: $S$ shared experts $\{{\cal SE}_{l,s}\}_s$, activated for every token \citep{liu2024deepseek}, and $R\!:=\!E\!-\!S$ routed experts $\{{\cal RE}_{l,e} \}_e$, activated in a token-driven manner. We use $K$ to denote the total number of activated experts, including the shared ones.

For $\h_l^n\in\mathbb{R}^d$, the representation of the $n$-th token entering the $l$-th MoE layer, the routed experts to be activated are selected by a linear projection $\G_l\in\mathbb{R}^{R\times d}$, known as the \emph{router}. The router produces a probability distribution over the routed experts, and those with the highest probabilities are activated. The activated routed experts are collected in ${\cal A}_l(\h_l^n)$:
\begin{align}\label{eq.moe-routing}
    p_{l,e} (\h_l^n)= \frac{\exp(\G_l [e,:] \h_l^n)}{\sum_{j=1}^R \exp(\G_l [j,:] \h_l^n)}, ~~~ {\cal A}_l(\h_l^n) = \topk_{K-S}   \big( \{p_{l,e} (\h_l^n) \}, e \in [R] \big).
\end{align}
Here, $\topk_{K-S}(\cdot)$ deterministically selects $K-S$ experts with the largest probabilities and $[R]:=\{1,\ldots,R\}$ denotes all the routed experts. 
The learned token representation of this layer combines the outputs of both shared and routed experts, where shared experts capture common features and routed experts learn specialized ones, i.e.,
\begin{align}\label{eq.moe}
	\h_{l+1}^n =\sum_{s=1}^S {\cal SE}_{l,s}(\mathbf{h}_l^n) + \sum_{e \in {\cal A}_l(\h_l^n) } \frac{p_{l,e} (\mathbf{h}_l^n)}{ Z(\h_l^n) } \cdot  {\cal RE}_{l,e}(\mathbf{h}_l^n)
\end{align}
where $Z(\h_l^n)=\sum_{j\in {\cal A}_l(\h_l^n) } p_{l,j} (\mathbf{h}_l^n)$ renormalizes the weights of the activated experts, as in \citep{yang2025qwen3} or $Z(\h_l^n)=1$ as in DeepSeek-MoE and Qwen1.5-MoE \citep{dai2024deepseekmoe,muennighoff2025olmoe, qwen-moe}. We follow the latter in this paper.
Vanilla MoEs without shared experts are recovered when $S=0$. For convenience, we omit the dependence of $p_{l,e}$ and ${\cal A}_l$ on $\h_l^n$, as well as the token index $n$ whenever clear from context.

Beyond shared experts, recent MoEs have also shifted toward fine-grained expert segmentation \citep{liu2024deepseek,yang2025qwen3,elango2026latentmoe}. For the same activation ratio, this trend favors configurations such as $R=64, K-S=8$ over $R=8, K-S=1$, enabling a combinatorially richer set of expert compositions and, hence, more expressive token representations. 
However, realizing such expressiveness through pretraining is far from straightforward, where
the challenge also stems from the \textit{combinatorial} expansion of the routing space. Identifying optimal expert combinations therefore requires more exploration. 
This motivates us to revisit exploration in MoE training.

\subsection{Shared experts weaken routing exploration}\label{sec.analysis}

This subsection reexamines the roles of shared and routed experts, reveals an asymmetry in their learning dynamics, and shows how shared-expert dominance can hinder exploration.

To isolate the impact of shared experts, we compare MoEs with and without them using matching datasets, tokenizers, and configurations. The models in this subsection have 121M activated and 498M total parameters, with $E=66$ and $K=8$. We set $S\!=\!2$ for shared-expert MoEs, while MoEs without shared experts amount to $S\!=\!0$. Details are deferred to App.\ref{sec.apdx.model-config}. We select checkpoints at $20\%$, $50\%$, and $100\%$ of pretraining to represent the early, middle, and final stages, respectively.

\textbf{Asymmetric training dynamics in shared-expert MoEs.} We first ablate individual experts to assess their importance. Specifically, we zero out each expert's output and measure the resulting perplexity change, $\Delta\mathsf{PPL}=\mathsf{PPL}_{\mathrm{after}}-\mathsf{PPL}_{\mathrm{base}}$. In Fig.\ref{fig.expert-ablation}, the two rightmost columns in the right-hand plots correspond to shared experts. Across all three checkpoints, ablating a shared expert leads to a substantially larger perplexity increase than ablating a routed expert. This shows that shared experts play a disproportionately important role in model performance.

\begin{wrapfigure}{r}{0.62\textwidth}
  \centering
  
 \vspace{-0.5cm}

  \begin{subfigure}{\linewidth}
    \centering
    \includegraphics[width=\linewidth]{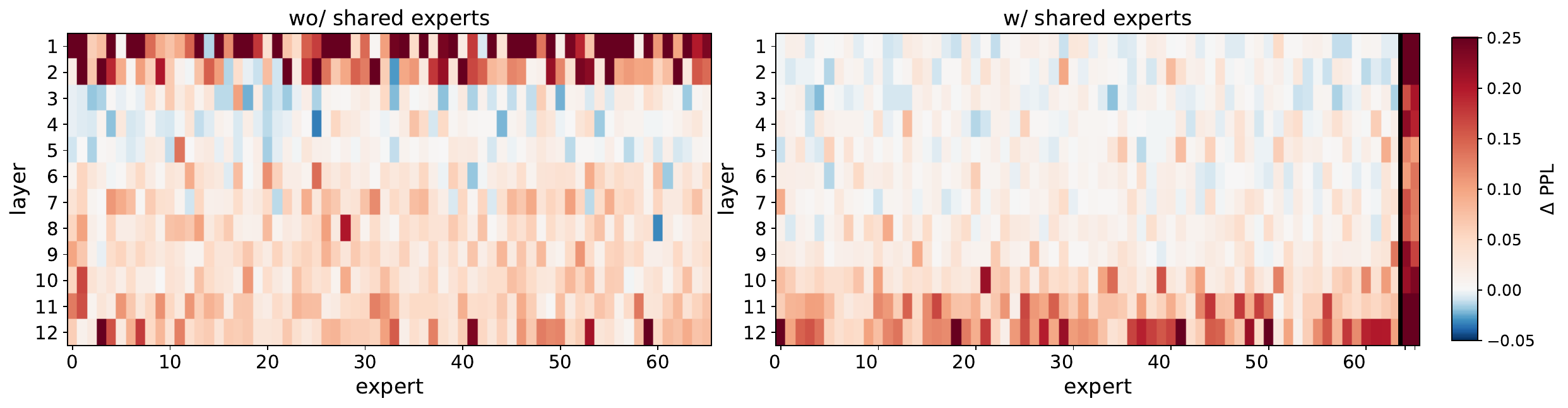}
    \vspace{-0.6cm}
    \caption{20\% of training}
    \label{fig:se-energy-20}
  \end{subfigure}

  \begin{subfigure}{\linewidth}
    \centering
    \includegraphics[width=\linewidth]{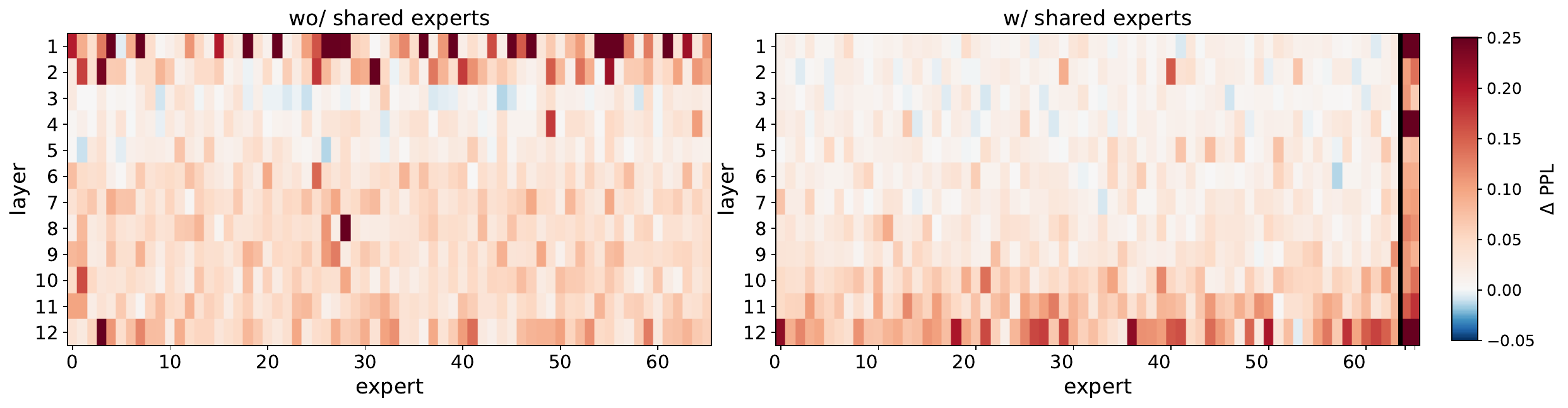}
    \vspace{-0.6cm}
    \caption{50\% of training}
    \label{fig:se-energy-50}
  \end{subfigure}

  \begin{subfigure}{\linewidth}
    \centering
    \includegraphics[width=\linewidth]{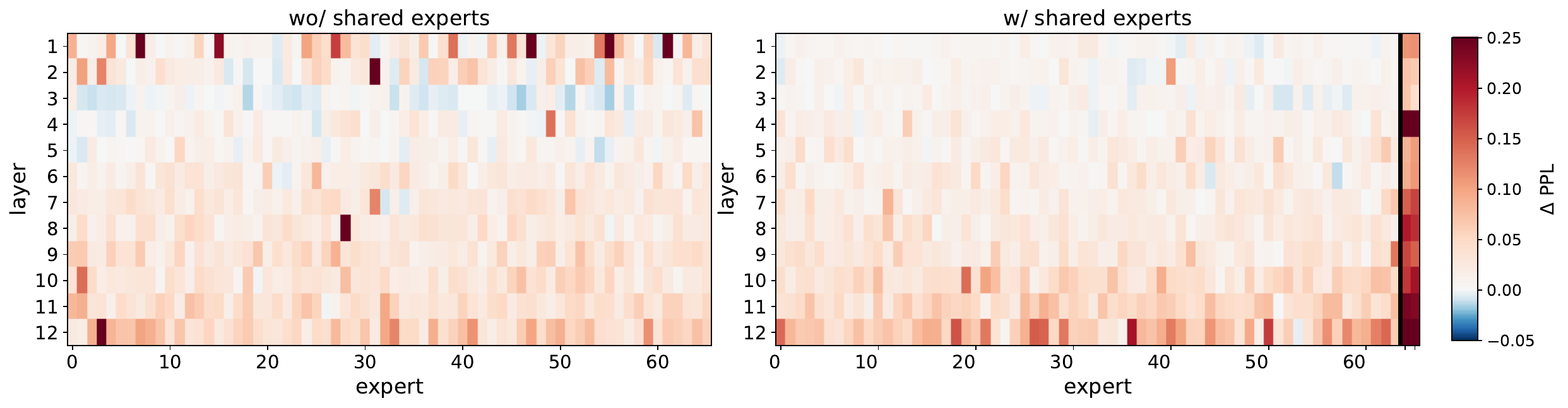}
    \vspace{-0.6cm}
    \caption{100\% of training}
    \label{fig:se-energy-100}
  \end{subfigure}
  \vspace{-0.6cm}
  \caption{Perplexity increase after ablating individual experts. In the right subplot, the two rightmost columns, separated by a solid line, correspond to shared experts.}
  \label{fig.expert-ablation}
  \vspace{-0.6cm}
\end{wrapfigure}

Another means to quantify the asymmetric roles of experts is through the so-termed \textit{routing criticality}, which replaces the learned router with uniform random routing and measures the resulting performance change.
 We apply this intervention one layer at a time and report the relative perplexity increase,
$\frac{\mathsf{PPL}_{\mathrm{after}}-\mathsf{PPL}_{\mathrm{base}}}{\mathsf{PPL}_{\mathrm{base}}}$. 
Larger values indicate greater dependence on the learned router. 
As shown in Fig.\ref{fig.routing-criticality}, uniform routing in the early and middle layers of the shared-expert MoE leads to a much smaller perplexity increase than in standard MoEs. 
In other words, \textit{shared experts in these layers capture such a large fraction of the learned representation that the layers are robust to random routing.}

Next, we examine how this asymmetry in expert roles shapes exploration. 
Exploration aims to discover alternative routes that produce diverse token representations. We therefore quantify exploration by measuring how much the learned representation changes under different routing decisions. 
Even before this analysis, the asymmetry between shared and routed experts suggests that exploration may be ineffective: if a highly exploratory routing policy, such as uniform routing, barely changes the output of a shared-expert layer, other routing is also unlikely to produce substantially different representations. We next confirm this intuition using two complementary variance metrics, one induced by experts and the other by routers.

\textbf{Shared experts suppress expert-induced variance.}
We first quantify how much an MoE layer’s output can vary when different experts are activated. To this end, we introduce the notion of \emph{expert variance}. Given an input token $\h_l$ and its learned output $\h_{l+1}$ in \eqref{eq.moe}, we define the expert variance as
\begin{align*}
	\mathsf{ExpertVar}_l = \frac{1}{R}\sum_{e=1}^R  \frac{\| \tilde{\h}_{l+1,e} - \mathbf{\h}_{l+1} \|^2}{\|\h_{l+1}\|^2} ~~~\text{where}~~~ \tilde{\h}_{l+1,e} = \sum_{s=1}^S {\cal SE}_{l,s}(\mathbf{h}_l) +  {\cal RE}_{l,e}(\h_l).
\end{align*}
In other words, $\tilde{\h}_{l+1,e}$ is the representation obtained if only a single expert $e$ is activated.
A larger $\mathsf{ExpertVar}_l$ indicates that changing routed experts can produce substantially different representations compared to $\mathbf{\h}_{l+1} $, indicating greater potential for exploration.
Moreover, $\mathsf{ExpertVar}_l$ bounds the relative change induced by \textit{any} routing policies from above if $Z(\h_l)=\sum_{j\in {\cal A}_l(\h_l) } p_{l,j} (\mathbf{h}_l)$. Let $\hat{\h}_{l+1}$ denote the representation obtained from any set of activated experts and routing probabilities. As shown in App.\ref{apdx.derive-ub}, we can prove that
	$\max_{\hat{\h}_{l+1}}  \|\hat{\h}_{l+1}-\h_{l+1}\|^2/ \|\h_{l+1}\|^2 \leq R \cdot \mathsf{ExpertVar}_l$. 
We evaluate $\mathsf{ExpertVar}_l$, averaged over tokens, for every MoE layer in Fig.\ref{fig.expert-var-norm}. 
Expert variance is smaller across all layers when shared experts are present. This is because shared experts dominate the output, while routed experts can only make relatively small variations. Consequently, shared-expert MoEs are less explorative even with different activated experts.

\begin{figure}[t]
  \centering
  \begin{subfigure}{0.32\textwidth}
    \centering
    \vspace{-0.1cm}
    \includegraphics[width=\textwidth]{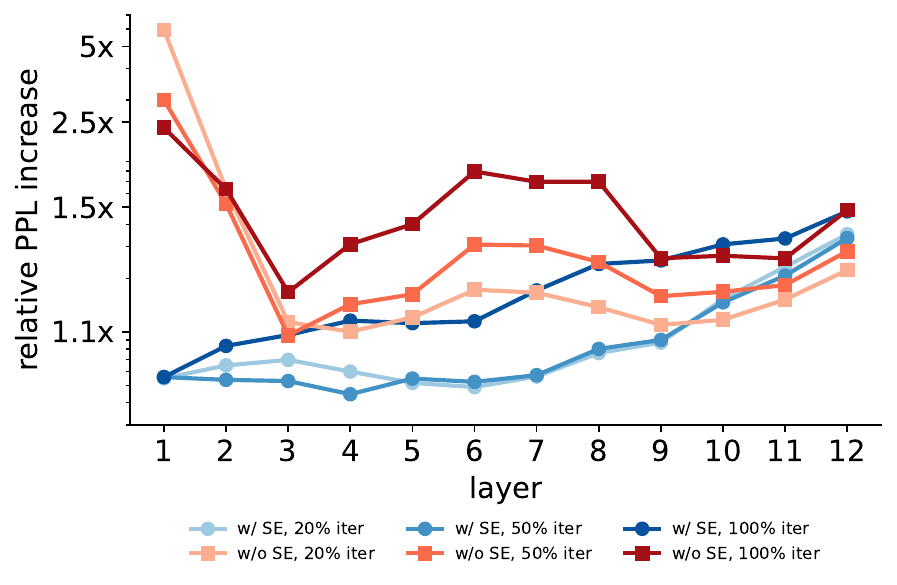}
    \caption{Routing criticality}
    \label{fig.routing-criticality}
  \end{subfigure}
  \label{fig.shared-vs-routed}
  \begin{subfigure}{0.32\textwidth}
    \centering
    \vspace{-0.1cm}
    \includegraphics[width=\textwidth]{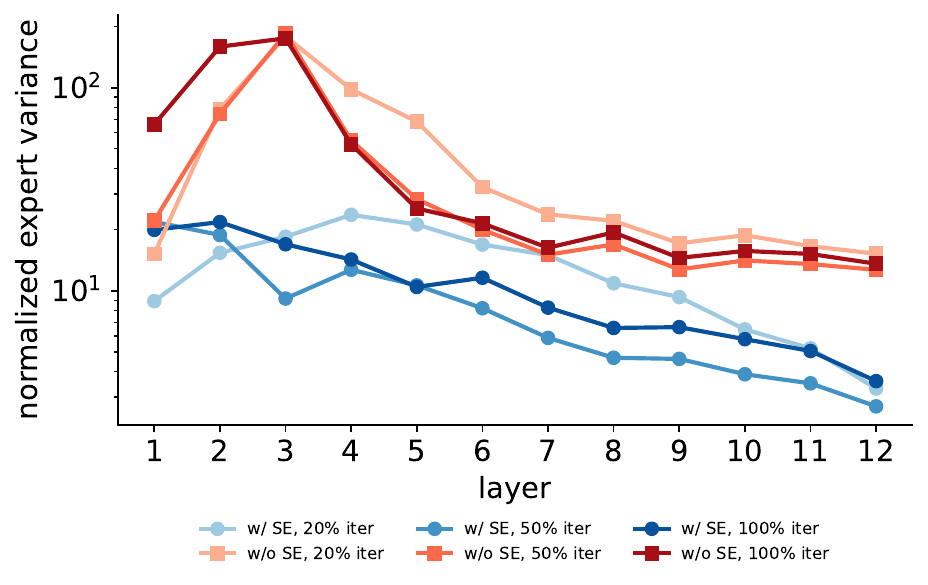}
    \caption{Expert variance}
    \label{fig.expert-var-norm}
  \end{subfigure}
  \begin{subfigure}{0.32\textwidth}
    \centering
    \vspace{-0.1cm}
    \includegraphics[width=\textwidth]{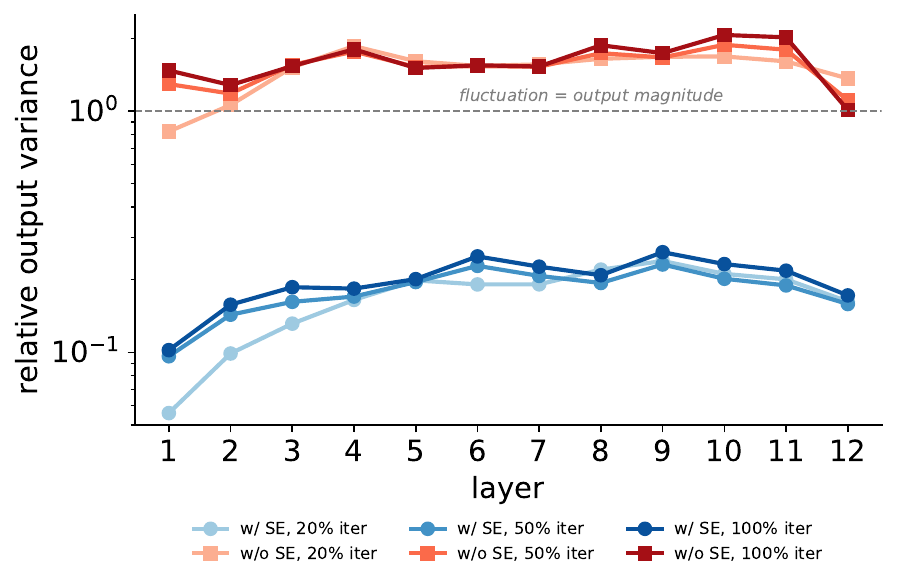}
    \caption{Router variance}
    \label{fig.router-var-norm}
  \end{subfigure}
  \caption{Shared experts make MoE routing less exploratory.}
 \end{figure}

\textbf{Shared experts reduce router-induced variance.} 
Next, we show that when routing decisions are sampled from the router’s learned distribution, shared-expert MoEs also produce less variable output features. To this end, we replace deterministic top-$(K-S)$ routing in \eqref{eq.moe-routing} with stochastic expert selection according to the learned routing probabilities. For each token, we independently sample $M=256$ expert sets and denote the resulting representation of the $m$-th sample by $\hat{\h}_{l+1}(m)$, where $m\in[M]$. We then define the router-induced variance as
\begin{align*}
	\mathsf{RouterVar}_l = \frac{1}{M}\sum_m  \frac{\| \hat{\h}_{l+1}(m) - \bar{\h}_{l+1} \|^2}{\|  \bar{\h}_{l+1} \|^2} ~~~ \text{where}~~~  \bar{\h}_{l+1} = \frac{1}{M}\sum_m \hat{\h}_{l+1}(m). 
\end{align*}
A larger $\mathsf{RouterVar}_l$ indicates that sampling according to the routing distribution is more likely to produce diverse features, and thus more explorative. Fig.\ref{fig.router-var-norm} reports $\mathsf{RouterVar}_l$ averaged over tokens. Because shared experts dominate the layer output, they reduce $\mathsf{RouterVar}_l$ by approximately an order of magnitude relative to a standard MoE. 
The stochastic routing above is equivalent to adding Gumbel noise with $\beta=1$ to the routing logits \citep{nie2021evomoe}. We further evaluate perturbed routing distributions with $\beta\in\{0.5, 2\}$. Results in App.\ref{apdx.num.gumbel} exhibit the same trend as Fig.\ref{fig.router-var-norm}. Thus, even under routing distributions perturbed from the learned one, shared-expert MoEs are less likely to produce diverse representations.

In summary, our diagnosis suggests that exploration faces inherent limitations in shared-expert MoEs.
\begin{tcolorbox}[colback=gray!5, colframe=black!75,
    boxrule=1.2pt, arc=2pt, left=6pt, right=6pt, top=5pt, bottom=5pt]
\textbf{Takeaway.} Shared experts can dominate representation learning, limiting representational diversity and thereby weakening routing exploration.
\end{tcolorbox}

\subsection{Consequences of the exploration dilemma}

The preceding discussions reveals a fundamental \textit{exploration dilemma} in shared-expert MoEs. 
On the one hand, the combinatorial growth of the routing space induced by fine-grained expert segmentation demands greater exploration to identify (near-)optimal expert combinations. 
On the other hand, shared experts reduce feature variation, thereby weakening exploration.
This mismatch gives rise to several practical inefficiencies. 
We reveal them by evaluating models ranging from 100M to 2.4B total parameters. We set $R \in \{64, 128\}$ and $S=2$ for shared-expert MoEs, and match the total number of experts for MoEs without shared experts. We vary the number of activated experts as $K\in\{6,8,10\}$. In this scenario, increasing $K$  improves expressiveness but also enlarges the exploration space.

\textbf{Inefficiency 1. Benefits of shared experts are eroded.} The first consequence is that shared-expert MoEs do not consistently improve upon standard MoEs. As shown in Fig.\ref{fig.scaling}, when $K$ is small and hence demanding less exploration, shared-expert MoEs outperform their non-shared counterparts. As $K$ increases, however, the benefits of shared experts diminish or even disappear.

\textbf{Inefficiency 2. Shared experts can diminish the expressive benefits of increasing $K$.} Across the four model scales in Fig.\ref{fig.scaling}, increasing $K$ from $6$ to $10$ yields perplexity gains of $\{2.23,1.32,$ $1.08,0.87\}$ for MoEs without shared experts. With shared experts, however, these gains drop to $\{0.29,-0.03,0.58,-0.03\}$, i.e., negative at two scales. These negative numbers indicate that additional expressiveness from activating more experts is not fully realized with shared experts.

\textbf{Inefficiency 3. Shared experts slow routing specialization.} As shown in Fig.\ref{fig.routing-criticality}, replacing the learned router with a uniformly random one leads to a steadily increasing relative perplexity gap across training stages for MoEs without shared experts. This is expected as routing becomes increasingly specialized. In contrast, shared-expert MoEs exhibit little increase in this gap between the $20\%$ and $50\%$ checkpoints, suggesting slower routing specialization.

\section{MirrorMoE for Pretraining Shared-Expert MoEs}

To reconcile this exploration dilemma and mitigate the inefficiencies discussed earlier, we propose MirrorMoE, a simple yet effective pretraining approach tailored to MoEs with shared experts.

\subsection{MirrorMoE shrinks the exploration space}

One solution to the exploration dilemma is to encourage exploration. However, this can introduce tradeoffs in MoE training, such as reduced stability and increased optimization variance, as further illustrated in App.\ref{apdx.num.pretrain}. We thus take the opposite approach: \textit{shrink the routing space to better match the model's exploration capability.} 

An intuitive way to view this is through the number of effective expert combinations, which scales with ${R\choose K_{\mathsf{eff}}}$. We formally define $K_{\mathsf{eff}}$ in Sec.\ref{sec.routing-visual}; for now, it can be understood as 
the number of functionally active experts 
induced by the routing method. The routing space is reduced when $K_{\mathsf{eff}}<K-S$. 
However, as shown by the ablation in Sec.\ref{sec.num.ablation}, naively reducing activated experts is suboptimal. MirrorMoE reduces $K_{\mathsf{eff}}$ by leveraging experts with low routing probabilities. It selects these low-probability experts in a token-adaptive manner but assigns them negligible aggregation weights. They therefore contribute little to the learned token representation and are functionally suppressed, reducing the number of experts that need to be explored. We detail this approach below.

\textbf{Mirror routing.}
Recall that $\G_l$ denotes the router weights at the $l$-th layer, and consider a token representation $\h_l\in\mathbb{R}^d$. Mirror routing activates experts according to a new routing probability
\begin{align}\label{eq.routing_prob}
    p_{l,e}  = (1 - \rho ) \underbrace{\frac{\exp(\G_l[e,:] \h_l)}{\sum_{j=1}^R \exp(\G_l [j,:] \h_l) }}_{:=\xi_{l,e}} + \rho \underbrace{ \frac{\exp(- \G_l [e,:] \h_l)}{\sum_{j=1}^R \exp(-\G_l [j,:] \h_l) }}_{:=\zeta_{l,e}}
\end{align}
where any $\rho\in[0,1]$ yields a valid probability distribution. The first term $\xi_{l,e}$ is the standard routing distribution and favors experts with larger logits. The second term $\zeta_{l,e}$ is obtained by negating the sign of $\G_l$, and therefore favors experts with smaller logits. This sign reversal motivates the name mirror routing. The convex combination of $\xi_{l,e}$ and $\zeta_{l,e}$ shifts routing probability towards both ends of the logit range. Experts with larger $p_{l,e}$ are activated
\begin{align}\label{eq.mirror-activated}
	{\cal A}_l(\h_l) = \topk_{K-S}   \big( \{p_{l, e} \}, e \in [R] \big).
\end{align}

\textbf{Mirror aggregation.}
This is the central mechanism for shrinking the exploration space, where we aggregate the selected experts using the probabilities $\xi_{l,e}$ instead of $p_{l,e}$
\begin{align}\label{eq.mirror-agg}
   \h_{l+1} = \sum_{s=1}^S {\cal SE}_{l,s}(\mathbf{h}_l) + \sum_{e \in {\cal A}_l} \frac{\xi_{l,e}}{Z(\h_l)}  {\cal RE}_{l,e} (\h_l).
\end{align}
Here, we set $Z(\h_l)=1$, analogous to \eqref{eq.moe} and consistent with the convention adopted throughout this work. Alternatively, one may use $Z(\h_l)=\sum_{j\in{\cal A}_l}\xi_{l,j}$ to normalize the routing weights over the activated experts.
Consequently, even when low-logit experts are activated, they contribute minimally to the representation $\h_{l+1}$ and are therefore functionally suppressed, reducing $K_{\mathsf{eff}}$. 
This reduction is \textit{adaptive}: it is determined by $\xi_{l,e}$ and can vary across tokens and layers, as visualized shortly in Sec.\ref{sec.routing-visual}.
Viewed jointly, MirrorMoE router $\G_l$ serves two complementary roles. As in standard routing, it governs exploitation by selecting experts with large logits. More importantly, it leverages small logits, which standard routing discards, to determine how many experts to suppress.

Lastly, we make two practical modifications to the pretraining procedure to further improve performance. First, because permanently suppressing experts could limit the expressiveness of the final model, we gradually anneal $\rho$ to $0$ over the course of training in \eqref{eq.routing_prob}. A cosine schedule is found effective in our experiments; see Sec.\ref{sec.num.ablation}. This yields a temporally strategic exploration scheme: the routing space gradually expands only after experts have developed more stable specialization.

Second, a modest additional improvement can be achieved by balancing the probabilities $\xi_{l,e}(\h_l^n)$ rather than the mirror-routing probabilities $p_{l,e}(\h_l^n)$. Specifically, modify the load-balancing loss to
\begin{align}\label{eq.load-balance-loss}
	\mathcal{L}_{\mathrm{bal}} = R\sum_{l=1}^L \sum_{e=1}^R f_{l,e} \bar{q}_{l,e}, ~~~ \text{with} ~~~f_{l,e}  = \frac{c_{l,e}}{N(K-S)}, ~~\bar{q}_{l,e} = \frac{1}{N} \sum_{n=1}^N \xi_{l,e}(\h_l^n)
\end{align}
where $N$ is the number of tokens and $c_{l,e}$ is the number of tokens routed to expert $e$ at $l$th layer. The intuition is that experts with small $\xi_{l,e}(\h_l^n)$ contribute little to the resulting token representations, even if activated. The load-balancing objective therefore focuses on the experts that are functionally active. As shown in App.\ref{apdx.load-balancing}, this modification improves performance. The complete MirrorMoE pretraining procedure is summarized in Alg.\ref{alg.mirrormoe}.

\begin{algorithm}[t]
\caption{MirrorMoE training iteration $t$}
\label{alg.mirrormoe}
\begin{algorithmic}[1]
\State \textbf{Forward:} $\rho_t \gets \operatorname{CosineAnneal}(\rho_0, t)$
\Statex Select experts for each token using \eqref{eq.routing_prob} and \eqref{eq.mirror-activated}, and aggregate their outputs using \eqref{eq.mirror-agg}
\Statex Compute $\mathcal{L}\gets\mathcal{L}_{\mathrm{task}}+\lambda_{\mathrm{bal}}\mathcal{L}_{\mathrm{bal}}$, where $\mathcal{L}_{\mathrm{bal}}$ is defined in \eqref{eq.load-balance-loss}
\State \textbf{Backward:} Backpropagate $\mathcal{L}$ and update the model parameters
\end{algorithmic}
\end{algorithm}

\subsection{Visualization of routing space reduction in MirrorMoE}\label{sec.routing-visual}

Next, we empirically show that MirrorMoE shrinks the exploration space early in training and ultimately learns a more confident routing policy, indicating stronger expert specialization. The notion of \emph{effective experts} is introduced to quantify this. 

\begin{definition}[(Token-wise) effective experts]\label{def.eff-experrts}
	For a given token $\h_{l}^n$, recall the mirror aggregation in \eqref{eq.mirror-agg}. The effective experts for this token is defined as 
    \begin{align*}
        K_{\mathsf{eff}}(\h_{l}^n) = \exp \Big(- \sum_{e \in {\cal A}_l(\h_{l}^n)} \bar{\xi}_{l,e}(\h_{l}^n) \log  \bar{\xi}_{l,e}(\h_{l}^n)\Big), ~~\text{where}~~  \bar{\xi}_{l,e}(\h_{l}^n) =  \frac{\xi_{l,e}(\h_{l}^n)}{\sum_{j  \in {\cal A}_l(\h_{l}^n) } \xi_{l,j}(\h_{l}^n)}.
    \end{align*}
\end{definition}
We omit the dependence on $\h_{l}^n$ for convenience.
This definition measures how many of the activated experts contribute substantially to the aggregated representation. In other words, $K_{\mathsf{eff}}(\h_{l}^n)$ is a smooth approximation to the number of functionally active experts, excluding shared ones. If $\bar{\xi}_{l,e}=1/(K-S)$ for every $e\in{\cal A}_l$, then $K_{\mathsf{eff}}=K-S$, indicating that all routed experts contribute significantly to the aggregation. Conversely, if $\bar{\xi}_{l,e}=1$ for one expert and zero for others, then $K_{\mathsf{eff}}=1$, indicating that the final token representation is determined by a single expert. 
Intuitively, if roughly $K_{\mathsf{eff}}$ experts contribute, the number of expert combinations to explore scales as ${R\choose K_{\mathsf{eff}}}$, and smaller $K_{\mathsf{eff}}$ corresponds to fewer ``active'' experts and a smaller effective exploration space.

We plot $ K_{\mathsf{eff}}$, averaged over all tokens and layers, against iterations in Fig.\ref{fig.effpath}. We train an MoE with 100M parameters and vary $K\in\{6,8,10\}$. Alongside $ K_{\mathsf{eff}}$, we report the minimum normalized aggregation weight among the activated experts, $\min_{e\in{\cal A}_l}\bar{\xi}_{l,e}$, also averaged over tokens and layers.
It can be seen that MirrorMoE has smaller $ K_{\mathsf{eff}}$ than the baseline during the first half of training. Meanwhile, $\min_{e\in{\cal A}_l}\bar{\xi}_{l,e}$ approaches $0$, showing that some activated experts contribute negligibly to the aggregated representation. Together, these results confirm that MirrorMoE reduces exploration space by functionally suppressing activated experts.
We further demonstrate that the adaptivity of the suppression in App.\ref{apdx.num.eff-exp} by plotting the layer-wise version of Fig.\ref{fig.effpath}.
Lastly, as $ K_{\mathsf{eff}}$ is the exponential entropy of the aggregation probabilities, its smaller value at the end of training in Fig.\ref{fig.effpath} shows that MirrorMoE learns a more concentrated routing policy, consistent with stronger expert specialization.

\begin{figure}[t]
    \centering
    \begin{tabular}{@{}c c c c@{}}
        & \includegraphics[width=0.30\textwidth]{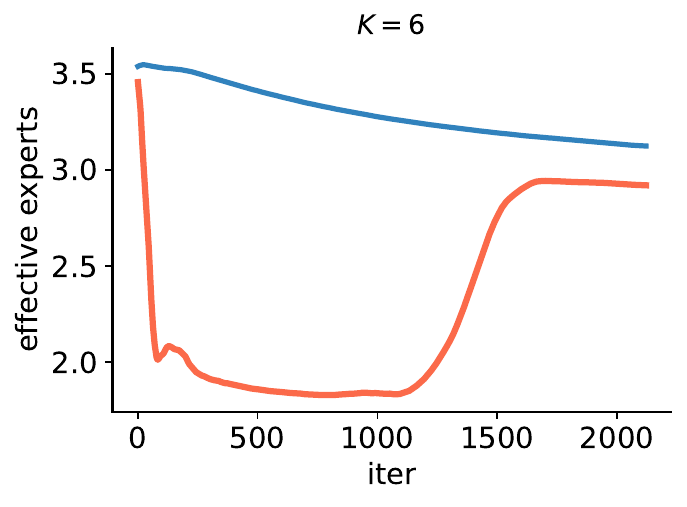}
        & \includegraphics[width=0.30\textwidth]{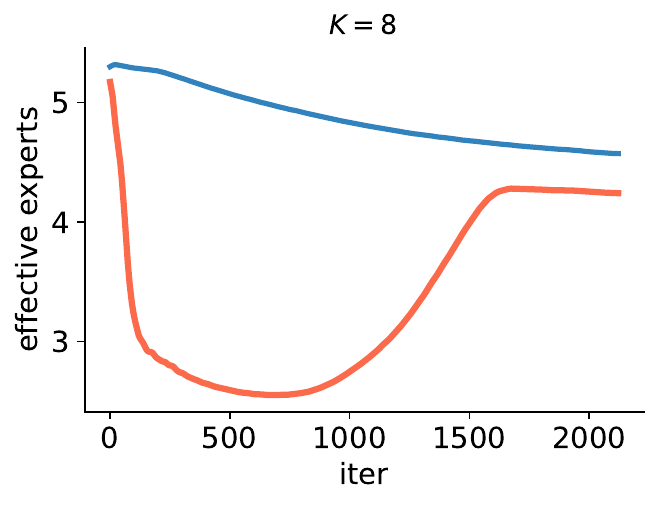}
        & \includegraphics[width=0.30\textwidth]{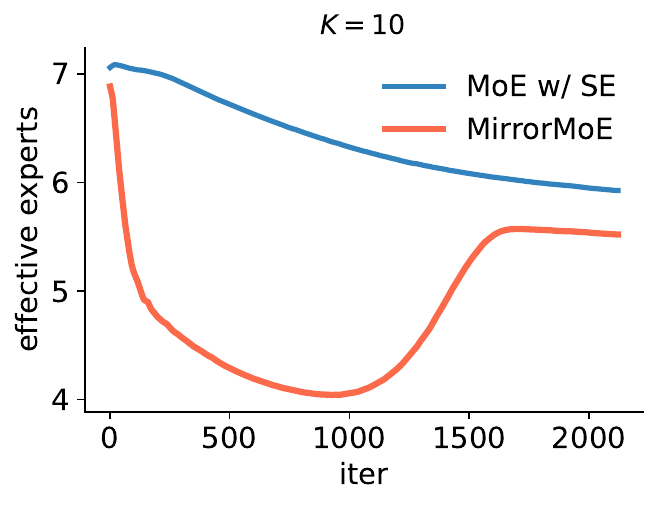} \\
        & \includegraphics[width=0.30\textwidth]{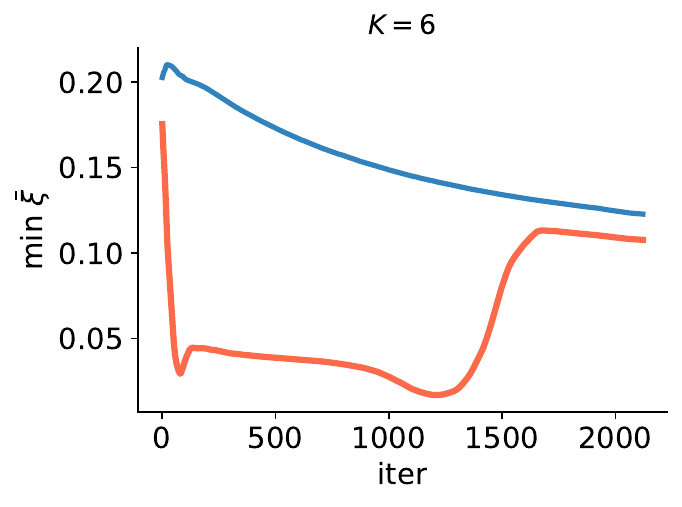}
        & \includegraphics[width=0.30\textwidth]{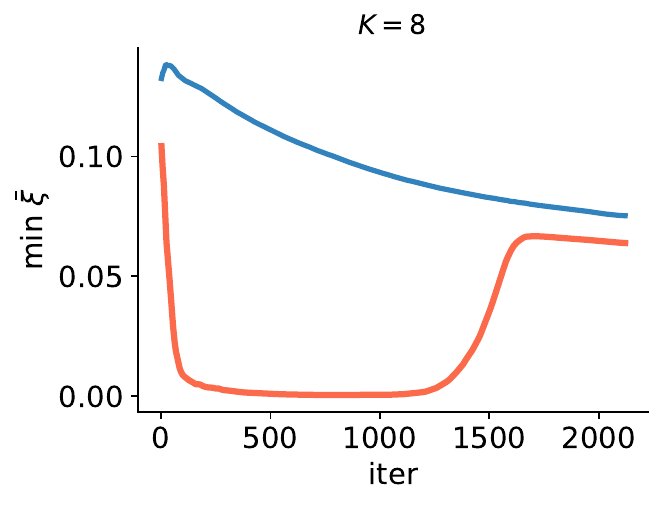}
        & \includegraphics[width=0.30\textwidth]{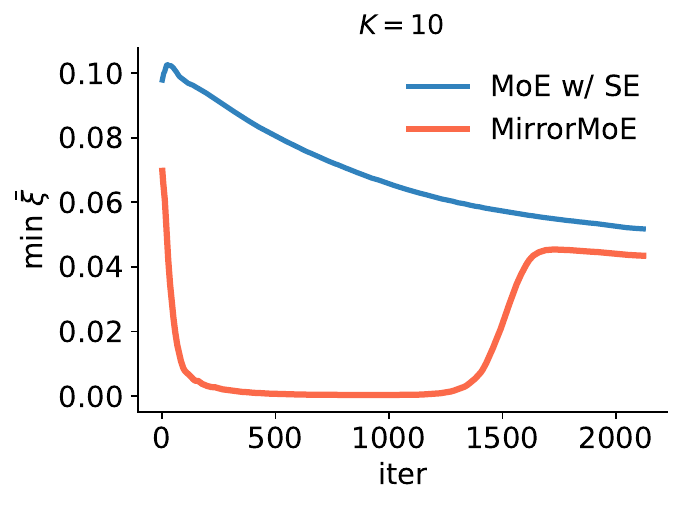} \\
    \end{tabular}
    \vspace{-0.3cm}
    \caption{Preview of how MirrorMoE in Alg.\ref{alg.mirrormoe} actively suppresses experts on a 100M MoE model. The rebound in $K_{\mathsf{eff}}$ results from annealing $\rho$.}
    \label{fig.effpath}
\end{figure}

\subsection{Discussions on adaptivity of MirrorMoE}

Lastly, we theoretically characterize the \textit{adaptivity} of MirrorMoE by deriving \textit{sufficient} conditions for both expert suppression and preservation at mirror routing. We begin with the former. When an expert has a sufficiently small $\xi_{l,e}$ and $\rho$ is chosen appropriately, MirrorMoE activates this expert but functionally removes it as negligible aggregation weight is assigned. For convenience, we omit the layer index throughout this subsection.

\begin{proposition}[Sufficient condition for expert suppression] \label{prop.path-pruning}

	For a given token, let $\{\xi_e\}$ and $\{p_e\}$ be probabilities obtained as \eqref{eq.routing_prob}.
	Let ${\cal A}^{p}:= {\cal A }$, i.e., \eqref{eq.mirror-activated}, denote the set of experts activated by MirrorMoE,
	and define ${\cal A}^{\xi}:=\topk_{K-S}(\{\xi_e\}_{e=1}^R)$ as the set that would be activated under the standard routing probabilities $\{\xi_e\}$.
	Let $e_{\min}:=\arg\min_{e\in[R]}\xi_e$ with $\xi_{e_{\min}}=\epsilon$, and let $j:=\arg\min_{e\in{\cal A}^{\xi}}\xi_e$ with $\xi_j=c/R$. As $\epsilon\to0$, $\rho>c/(1+c)$ is sufficient for $e_{\min}\in{\cal A}^{p}\setminus{\cal A}^{\xi}$; that is, MirrorMoE activates this expert while functionally suppressing its contribution during aggregation.

\end{proposition}

This explains the reduced $K_\mathsf{eff}$ as visualized in Fig.\ref{fig.effpath}.
Next, we show the token adaptivity of MirrorMoE by identifying a regime in which no expert is suppressed. Specifically, when the router is uncertain about its choices, MirrorMoE behaves similarly to standard routing.

\begin{proposition}[Sufficient condition for expert preservation]\label{prop.no-path-pruning}
Let $\xi_{\max}:=\max_{e\in[R]}\xi_e$, $\xi_{\min}:=\min_{e\in[R]}\xi_e$, and $\kappa\!:=\!\sqrt{(1-\rho)/\rho}$. If $\xi_{\max}/\xi_{\min}\leq\kappa$, ${\cal A}^{p}\!=\!{\cal A}^{\xi}$. Moreover, with $\bar{\xi}_e$ defined in Def.\ref{def.eff-experrts}, it satisfies $\bar{\xi}_e\geq 1/((K\!-\!S)\kappa), \forall e\in{\cal A}^{p}$, i.e., no selected expert is functionally suppressed.
\end{proposition}

\textbf{Another interpretation for MirrorMoE.} 
The above discussions provide an alternative interpretation of MirrorMoE beyond exploration. 
In MoEs with large $E$ and $K$, not every token may require the full expressiveness from all activated experts. 
If a token can be learned well using fewer experts, the capacity of remaining experts can be devoted to other tokens.
MirrorMoE can therefore be viewed as an approach for tokens to adaptively determine their number of activated experts, where this adaptivity arises from the expert suppression and preservation shown in Props. \ref{prop.path-pruning} and \ref{prop.no-path-pruning}, respectively.

\section{Numerical Results}\label{sec.num}

This section evaluates the performance of MirrorMoE. Our implementation builds on \citep{kang2025flame} and uses the Megatron-LM framework \citep{megatron-lm} with the DCLM dataset \citep{li2024datacomp}. All experiments are conducted on NVIDIA H100 GPUs. 
\faGithub\ Code is available at
\url{https://github.com/BingcongLi/MirrorMoE/tree/main}.

\subsection{Pretraining performance of MirrorMoE}

We evaluate four shared-expert MoEs ranging from 100M to 2.4B total parameters, with $R \in\{64,128\}$ experts. For each token, $6$ routed experts and $S=2$ shared experts are activated, yielding $K=8$. Activated parameters account for $16\%$ to $38\%$ of the total parameters across models. The model configurations are summarized in Tab.\ref{tab.numerical.test_ppl} and detailed in App.\ref{sec.apdx.model-config}.

\begin{table}[t]
\centering
\vspace{-0.2cm}
\caption{Pretraining performance of MirrorMoE across scales. D/S in the depth column denotes dense/sparse Transformer layers.}
\vspace{-0.25cm}
\small
\setlength{\tabcolsep}{10pt}
\renewcommand{\arraystretch}{1.15}
\resizebox{\textwidth}{!}{
\begin{tabular}{ccccccc}
\toprule
\textbf{Model}
& {Experts}
& {Depth}
& \multicolumn{2}{c}{\textbf{PPL} ($\downarrow$)}
& \multicolumn{2}{c}{\textbf{iteration time (ms)} ($\downarrow$) }   
\\
\cmidrule(lr){4-5}
\cmidrule(lr){6-7}
{Total-activated params.}
& \textbf{$K$/$R$/$S$}
& 
& {Vanilla}
& \cellcolor{blue!10} {\textbf{MirrorMoE}}
& {Vanilla}
& \cellcolor{blue!10} {\textbf{MirrorMoE}} \\
\midrule
MoE-100M-A38M
& 8/64/2
& 1D + 8S
& 54.78
& \cellcolor{blue!10}  \textbf{51.77}
& 7871.3 $\pm$ 65.9
& \cellcolor{blue!10} \textbf{7844.4} $\pm$ 	85.6 \\

MoE-498M-A121M
& 8/64/2
& 1D + 12S
& 28.31 
& \cellcolor{blue!10}  \textbf{27.36}
& 13550.0 $\pm$ 63.0
& \cellcolor{blue!10} \textbf{13426.7} $\pm$ 83.3 \\

MoE-625M-A100M
& 8/128/2
& 1D + 8S
& 28.46
& \cellcolor{blue!10}  \textbf{27.84}
& 9527.5 $\pm$ 45.4
& \cellcolor{blue!10} \textbf{9462.5} $\pm$ 61.2 \\

MoE-2.4B-A270M
& 8/128/2
& 1D + 15S
& 21.17
& \cellcolor{blue!10} \textbf{20.50}
& 25711.1 $\pm$ 176.2
& \cellcolor{blue!10} \textbf{25635.4} $\pm$ 109.4 \\
\bottomrule
\end{tabular}
}
\label{tab.numerical.test_ppl}  
\vspace{-0.2cm}
\end{table}

\textbf{Perplexity.} We first compare the test perplexity of pretrained MoE models in Tab.\ref{tab.numerical.test_ppl}.
MirrorMoE consistently achieves lower perplexity across all model scales.
Notably, MoE-498M-A121M and MoE-625M-A100M use the same hidden dimensions and differ primarily in numbers of layers and experts. Since the number of end-to-end routing paths scales as ${E-S \choose K-S}^{L}$, where $L$ denotes the number of sparse layers, MoE-498M-A121M has a larger routing space under our configurations. MirrorMoE also yields a larger performance gain on this model, suggesting that shrinking the routing space is particularly beneficial when the exploration dilemma is more pronounced.

\begin{table*}[t]
\centering
\caption{Downstream performance of MoE models trained with MirrorMoE. }
\label{tab.downstream}
\vspace{-0.3cm}
\setlength{\tabcolsep}{5.5pt}
\renewcommand{\arraystretch}{1.1}
\resizebox{0.98\textwidth}{!}{
\begin{tabular}{lcccccccc}
\toprule
\textbf{Model}
& \textbf{Method}
& {ARC-E}
& {OBQA}
& {HellaSwag}
& {PIQA}
& {BoolQ}
& {SIQA}
& {Avg} ($\uparrow$)
\\
\midrule

\multirow{2}{*}{MoE-100M-A38M}
& Vanilla
& 32.11 & 26.40 & 26.62 & 57.56 & 42.66 & 34.29 & 36.61 \\
&\cellcolor{blue!10}MirrorMoE
& \cellcolor{blue!10}34.22 & \cellcolor{blue!10}26.00 & \cellcolor{blue!10}25.76 & \cellcolor{blue!10}58.00 & \cellcolor{blue!10}42.48  & \cellcolor{blue!10}34.49 & \cellcolor{blue!10}\textbf{36.82}  \\
\midrule

\multirow{2}{*}{MoE-498M-A121M}
& Vanilla
& 39.44 & 26.80 & 29.87 & 61.43 & 58.65 & 36.69 & 42.15 \\
& \cellcolor{blue!10}MirrorMoE
& \cellcolor{blue!10}41.04 & \cellcolor{blue!10}27.80 & \cellcolor{blue!10}30.33 & \cellcolor{blue!10}64.31 & \cellcolor{blue!10}59.20 & \cellcolor{blue!10}36.95 & \cellcolor{blue!10}\textbf{43.27} \\
\midrule

\multirow{2}{*}{MoE-625M-A100M}
& Vanilla
& 41.12 & 28.20 & 29.97 & 61.10 & 61.47  & 37.36 & 43.20 \\
& \cellcolor{blue!10}MirrorMoE
& \cellcolor{blue!10}41.37 & \cellcolor{blue!10}29.80 & \cellcolor{blue!10}30.23 & \cellcolor{blue!10}62.08 & \cellcolor{blue!10}61.16  & \cellcolor{blue!10}36.59 & \cellcolor{blue!10}\textbf{43.53} \\
\midrule

\multirow{2}{*}{MoE-2.4B-A270M}
& Vanilla
& 45.37 & 30.00 & 36.63 & 66.59 & 54.50  & 38.33 & 45.23 \\
& \cellcolor{blue!10}MirrorMoE
& \cellcolor{blue!10}46.63 & \cellcolor{blue!10}32.40 & \cellcolor{blue!10}37.40 & \cellcolor{blue!10}67.95 & \cellcolor{blue!10}58.53  & \cellcolor{blue!10}38.28 & \cellcolor{blue!10}\textbf{46.87} \\
\bottomrule
\end{tabular}
}
\vspace{-0.2cm}
\end{table*}

\textbf{Runtime.} MirrorMoE slightly reduces the per-iteration training time despite introducing an additional softmax operation in \eqref{eq.routing_prob} at each layer. It suggests improved load balance. Note that MoE-100M-A38M is trained on two GPUs, whereas all other models are trained on four.

\begin{wrapfigure}{r}{0.3\columnwidth}
    \centering
   \vspace{-0.3cm}
    \includegraphics[width=\linewidth]{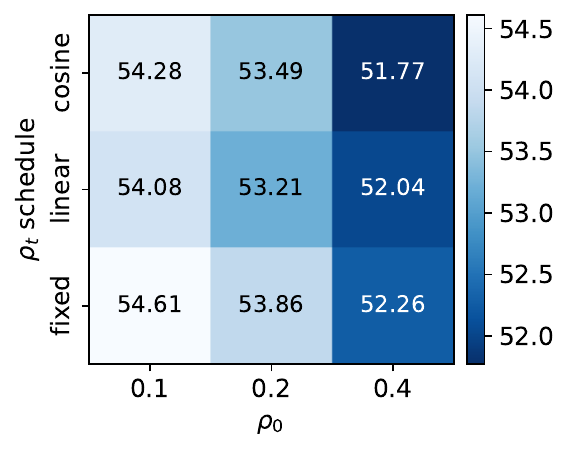}
    \vspace{-0.8cm}
    \captionof{figure}{Hyperparameters.}
    \label{fig.hyperparams}

    \vspace{-0.15cm}

    \captionof{table}{Comparison with non-adaptive schedules.}
    \resizebox{0.99\linewidth}{!}{
    \begin{tabular}{lc}
        \toprule
        Method & Test PPL ($\downarrow$) \\
        \midrule
        MoE & 54.78 \\
        Linear scheduled $\widetilde{K}_t$ & 54.91 \\
        Cosine scheduled $\widetilde{K}_t$ & 54.19 \\
        MirrorMoE & \textbf{51.77} \\
        \bottomrule
    \end{tabular}
    }
    \label{tab.num.adaptivity}
    \vspace{-1.5cm}
\end{wrapfigure}

\textbf{Downstream tasks.} We further demonstrate that the gains of MirrorMoE translate into improved downstream performance on six commonsense reasoning tasks in Tab.\ref{tab.downstream}. For the 2.4B model, MirrorMoE improves the average performance by 1.64.

\subsection{Comparison under enlarged routing space}

Building on the same model configuration, we vary the number of activated experts as $K\in\{6,8,10\}$. Increasing $K$ enables richer expert combinations, thereby enlarging both the exploration space and model expressiveness. As shown in Fig.\ref{fig.scaling}, the standard shared-expert MoE is outperformed by its non-shared counterpart at large $K$. In contrast, MirrorMoE mitigates the exploration dilemma, enabling the shared-expert model to consistently outperform the non-shared MoE across all tested cases. More importantly, the perplexity of MirrorMoE consistently decreases as $K$ increases, suggesting that the additional expressiveness enabled by larger $K$ is realized.

\subsection{Additional experiments}\label{sec.num.ablation}

We use a MoE-100M-A38M with $K=8$ (2 shared and 6 routed experts) for ablation study.

\textbf{Hyperparameters.} Fig.\ref{fig.hyperparams} compares constant, linear, and cosine schedules for $\rho_t$ in \eqref{eq.routing_prob}, together with initial values $\rho_0\in\{0.1,0.2,0.4\}$. A cosine schedule with $\rho_0=0.4$ achieves the best performance and is therefore used in all preceding experiments.

\textbf{Benefit of adaptivity.} We demonstrate the importance of token/layer adaptivity in MirrorMoE by comparing it with two simple approaches that also have smaller exploration space. Both begin with fewer activated experts and progressively increase their number. Specifically, we set $K=\mathsf{round}(\widetilde{K}_t)$, where $\widetilde{K}_t$ increases from $4$ to $8$ following either a linear or cosine schedule. Here, $\widetilde{K}_t$ includes two shared experts, and the initial value $\widetilde{K}_0=4$ is informed by the number of effective experts in Fig.\ref{fig.effpath}. These baselines use the same $K$ for all tokens and layers and therefore lack adaptivity compared to MirrorMoE. As shown in Tab.\ref{tab.num.adaptivity}, neither matches the performance of MirrorMoE.

\textbf{MirrorMoE is tailored to shared-expert MoEs.} We additionally apply MirrorMoE to MoEs \textit{without} shared experts to further validate our motivation based on the exploration dilemma induced by shared experts. As shown in Tab.\ref{tab.numerical.non-shared-moe} in App.\ref{apdx.moe-ns}, MirrorMoE performs comparably to vanilla training. This is consistent with our analysis: without shared experts, the exploration dilemma is less pronounced, leaving less need to shrink the exploration space.

\section{Conclusion}

This work identifies an exploration dilemma in MoEs that combine fine-grained expert segmentation with shared experts. This tension can prevent the additional expressiveness enabled by fine-grained segmentation from being fully realized while also diminishing the benefits of shared experts. To address this, we propose MirrorMoE, which adaptively shrinks the exploration space to better match the model's exploration capability. MirrorMoE consistently improves performance across scales.

\textbf{Future direction.} Our current approach addresses the exploration dilemma by shrinking the exploration space. A promising future direction is to improve the routing strategy itself for more efficient exploration.


\newpage
\bibliography{refs}
\bibliographystyle{plainnat}

\newpage
\appendix
\begin{center}
	\Large \textbf{Appendix}
\end{center}

\section{Additional related work}\label{apdx.related-work}

\textbf{Expressiveness does not necessarily translate directly into performance.} Beyond those we have discussed earlier in Sec.\ref{sec.related-work}, additional works similarly suggest that the expressive capacity of neural networks may not be fully utilized in practice. For example, \citep{depth-efficient} shows that many deeper layers in Llama 3.1 70B behave close to identity mappings. Since identity mappings are available ``for free'', this motivates redesign of residual connections for scaling deeper \citep{zhang2026ancre}. 

\textbf{MoE systems.} A complementary line of work focuses on efficient MoE systems, including faster and more principled expert parallelism, fine-grained all-to-all communication, and topology-aware communication strategies; see e.g., \citep{hwang2023tutel,liu2025netmoe,zhang2025comet}. These methods aim to improve hardware utilization, communication efficiency, and end-to-end throughput. MirrorMoE is largely orthogonal to these system-level optimizations and can be integrated with them.

\textbf{Routing in MoEs.} A broad line of work studies alternative routing mechanisms for MoEs. Hash Layers \citep{roller2021hash} use hash-based assignments, while BASE Layers \citep{lewis2021base} and expert-choice routing \citep{zhou2022mixture} promote balanced expert utilization through carefully designed token assignments, at the cost of not being strictly autoregressive. \citet{clark2022unified} study scaling laws for routed models. Other approaches replace hard top-$K$ selection with differentiable routing mechanisms \citep{muqeeth2023soft,wang2025remoe,vahidi2026dirmoe}. Several works post-train or fine-tune routers to enable token-adaptive computation \citep{yue2024ada,zeng2024adamoe}, while \citet{kang2025exploration} uses reinforcement learning to improve routing decisions. During pretraining, noisy routing has also been explored to encourage broader expert exploration: \citet{shazeer2017outrageously} introduce Gaussian noise, whereas \citet{nie2021evomoe} use Gumbel noise. Routing has also been explicitly studied for multimodal models \citep{wei2025routing}. In contrast to approaches that encourage broader exploration, MirrorMoE adaptively suppresses experts to shrink the exploration space.
The work of \citet{liang2026not} studies shared experts from an inference perspective, showing that they can hinder expert offloading, partly because shared experts reduce the importance of routed experts. However, their focus is on routing behavior at inference time, whereas we study how shared experts affect routing throughout pretraining.

\section{Experimental details}

\subsection{Training details}\label{sec.apdx.model-config}

The model configurations for the shared-expert MoE baseline and MirrorMoE are summarized in Tab.\ref{tab.shared-expert-moe-configs}. Following \citet{kang2025flame}, we adopt a Llama-style MoE architecture. We also follow DeepSeekMoE \citep{dai2024deepseekmoe,liu2024deepseek} by using a dense FFN in the first Transformer layer, so that the first sparse layer receives sufficiently developed token representations for routing. MoE-100M-A38M is trained on two H100 GPUs, while all other models are trained on four. DCLM dataset \citep{li2024datacomp} is adopted.
We use AdamW for optimization. The learning rate is selected from $\{1,3,5\}\times 10^{-4}$ based on performance after 500 training steps on a held-out validation set. We adopt a warmup-stable-decay (WSD) learning-rate schedule. The global batch size is $1024$, with gradient accumulation used to reduce memory consumption.

The model analyzed in Sec.\ref{sec.analysis} is MoE-498M-A121M.

For MoEs without shared experts, we convert the shared experts into routed experts to preserve the same total parameter count. Due to implementation constraints in Megatron-LM, this requires a different expert layout across GPUs, i.e., a different expert-parallelism configuration, and consequently a different number of GPUs. Specifically, Megatron-LM requires the number of routed experts to be divisible by the expert-parallelism size. Therefore, configurations with 66 or 130 routed experts use three or five GPUs, respectively. The global batch size consequently differs across configurations, while the total number of training tokens is kept approximately the same to ensure a fair comparison.

\subsection{Router Variance under Different Levels of Gumbel Noise}\label{apdx.num.gumbel}

When evaluating $\mathsf{RouterVar}_l$, we replace deterministic top-$K$ (for non-shared MoEs) or top-$(K-S)$ (for shared-expert MoEs) routing with stochastic sampling from the learned routing distribution. Such sampling is equivalent to perturbing the routing logits with unit-scale Gumbel noise ($\beta=1$) and then applying top-$K$ selection \citep{kool2019stochastic,nie2021evomoe}. We further extend this evaluation by defining $\mathsf{RouterVar}_l^\beta$, where the routing logits are perturbed by Gumbel noise parameterized by other $\beta$s. We consider $\beta\in\{0.5,2\}$ and report the resulting $\mathsf{RouterVar}_l^\beta$ for MoE-498M-A121M model
in Fig.\ref{apdx.fig.routing-variance-with-gumbel-beta}. The same trend persists across these perturbation levels, further showing that shared experts consistently reduce routing-induced variance.

\begin{table}[t]
\caption{Model configurations used in our experiments.}
\centering
\resizebox{\textwidth}{!}{
\begin{tabular}{lcccc}
\toprule
\textbf{Parameter} &
\textbf{MoE-100M-A38M} &
\textbf{MoE-498M-A121M} &
\textbf{MoE-625M-A100M} &
\textbf{MoE-2.4B-A270M} \\
\midrule
Dense / MoE layers          & 1 / 8 & 1 / 12 & 1 / 8 & 1 / 15 \\
Layer layout                 & 1D + 8S &
                              1D + 12S &
                              1D + 8S &
                               1D + 15S \\
Hidden size ($d_{\mathrm{model}}$) & 256 & 512 & 512 & 768 \\
Dense FFN hidden            & 1368 & 2736 & 2736 & 3648 \\
MoE FFN hidden (per expert) & 176 & 352 & 352 & 512 \\
Shared expert FFN (2$\times$) & 352 & 704 & 704 & 1024 \\
\# Routed experts           & 64 & 64 & 128 & 128 \\
Top-$k$ (activated)         & 6 & 6 & 6 & 6 \\
Sequence length             & 2048 & 2048 & 2048 & 2048 \\
Load balancing 				& 0.1 & 0.1 & 0.1 & 0.1 \\
\# Training tokens		    & 4.4B & 7.6B & 7.6B & 12.8B \\
\bottomrule
\end{tabular}
}
\label{tab.shared-expert-moe-configs}
\end{table}

\begin{figure}[t]
  \centering
  \begin{subfigure}{0.49\textwidth}
    \centering
    \vspace{-0.1cm}
    \includegraphics[width=\textwidth]{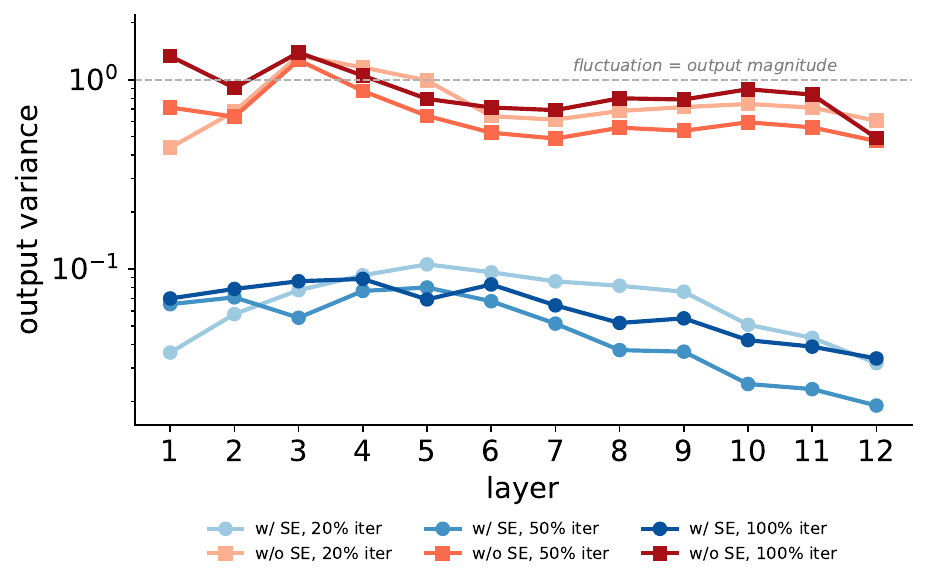}
    \caption{Gumbel noise with $\beta=0.5$}
    \label{fig.ablation-aux-loss}
  \end{subfigure}
  \begin{subfigure}{0.49\textwidth}
    \centering
    \vspace{-0.1cm}
    \includegraphics[width=\textwidth]{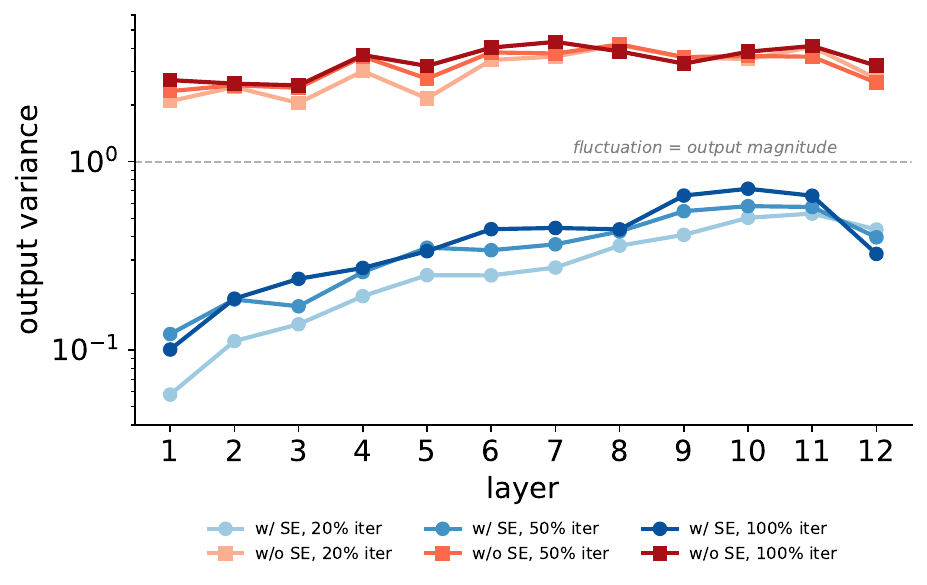}
    \caption{Gumbel noise with $\beta=2$}
    \label{fig.ablation-hyper-param}
  \end{subfigure}
  \caption{Router variance under different levels of Gumbel perturbation.}
  \label{apdx.fig.routing-variance-with-gumbel-beta}
\end{figure}

\subsection{Increasing Exploration during Pretraining Can Be Tricky}\label{apdx.num.pretrain}

An alternative way to address the exploration dilemma is to encourage greater exploration during pretraining. However, this can be challenging for MoEs because stochastic routing may introduce unstable router dynamics and increased training variance. To examine this tradeoff, we consider two alternative routing strategies under $K=8$. First, we replace deterministic top-$(K-S)$ routing with stochastic sampling of $K-S$ experts, denoted as Sto-$(K-S)$ in Tab.\ref{tab.stok}. Second, we deterministically select the top-$2$ experts and stochastically sample the remaining $K-S-2$ experts, denoted as Top-$2$ + Sto-$(K-S-2)$.

As shown in Tab.\ref{tab.stok}, on a MoE-100M-A38M model, deterministic top-$(K-S)$ routing performs best, followed by Top-$2$+Sto-$(K-S-2)$, while fully stochastic routing performs worst. These results suggest that the additional instability introduced by exploration can outweigh its potential benefits during pretraining. For this reason, MirrorMoE takes the opposite approach by shrinking the exploration space.

\begin{table}[t]
\caption{Intrinsic difficulty of increasing exploration}
\centering
\small
\begin{tabular}{lccc}
\toprule
		 & Top-$(K-S)$ & Sto-$(K-S)$ & Top-2 with Sto-$(K-S-2)$ \\
\midrule
PPL      & 54.78   & 55.74   & 55.38 \\
\bottomrule
\end{tabular}
\label{tab.stok}
\end{table}

\subsection{Layer-wise effective experts}\label{apdx.num.eff-exp}

Extending Fig.\ref{fig.effpath}, Fig.\ref{apdx.fig.effpaths} reports the layer-wise number of effective experts, averaged over tokens when training an MoE-100M-A38M model. For example, when $K=8$ in Fig.\ref{fig.K8}, Layers 1 and 2 maintain roughly three effective experts during the first half of training, whereas Layer 3 suppresses more aggressively and approaches a single effective expert. In contrast, Layers 7 and 8 are more conservative, retaining around four effective experts. This confirms that MirrorMoE shrinks the exploration space in a layer-adaptive manner.

\begin{figure}[t]
  \centering
  \begin{subfigure}{0.99\textwidth}
   \includegraphics[width=0.98\textwidth]{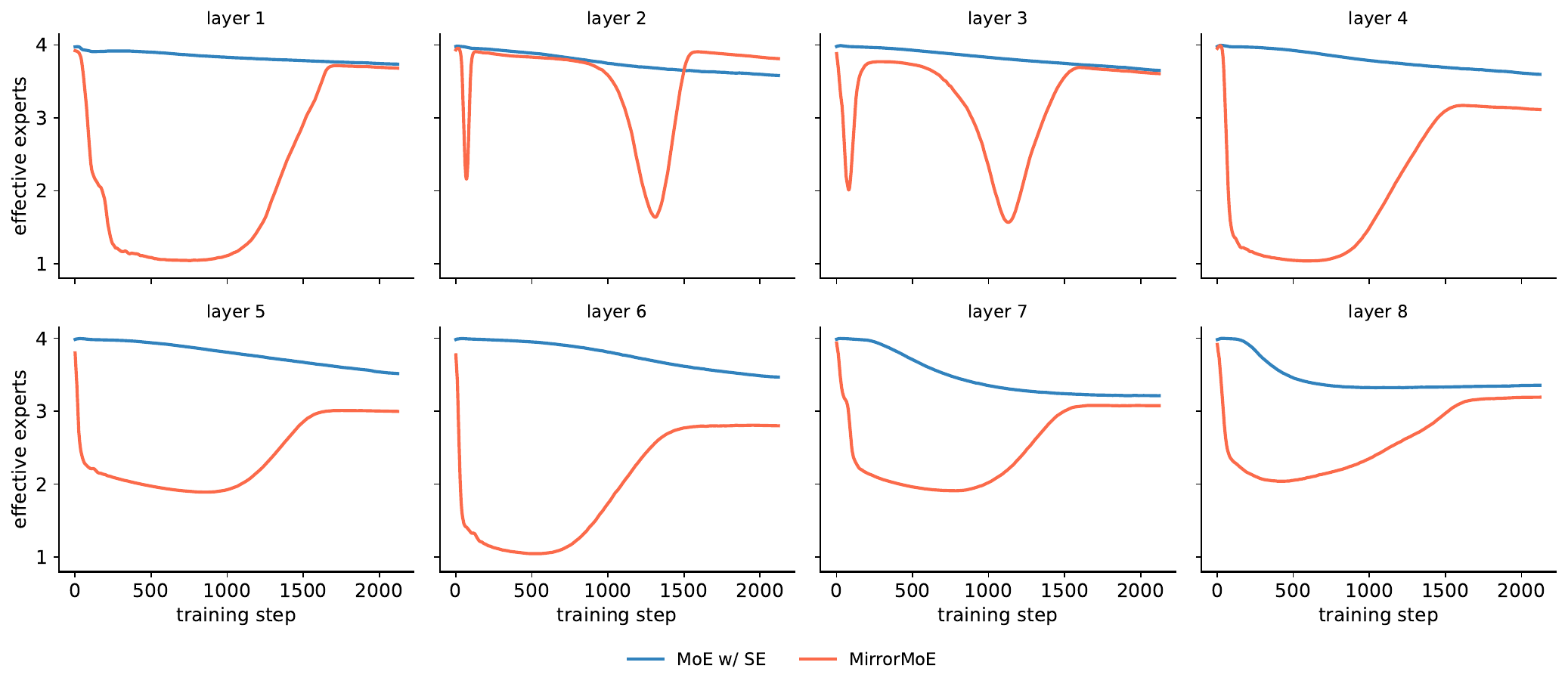}
    \caption{$K=6$}
    \label{fig.K6}
   \end{subfigure}
   
   \begin{subfigure}{0.99\textwidth}
   \includegraphics[width=0.98\textwidth]{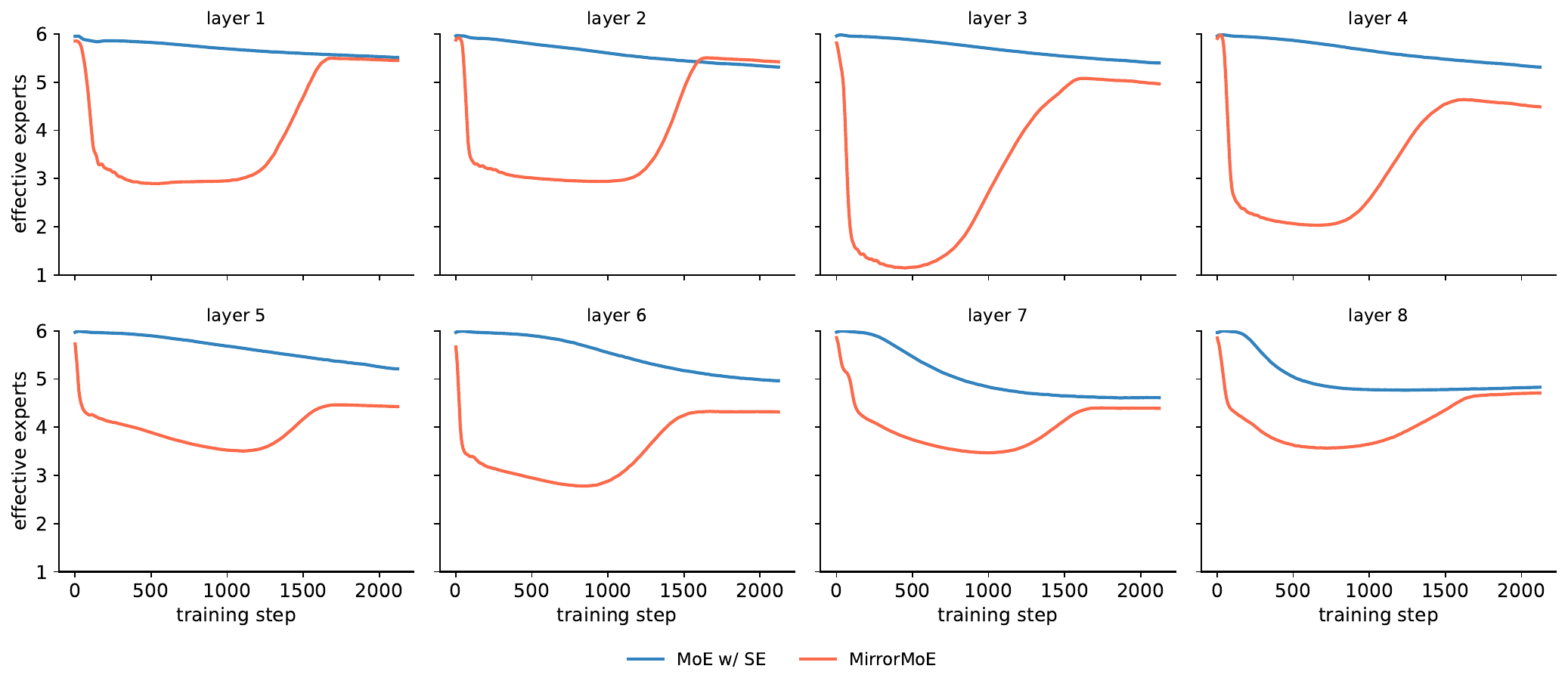}
    \caption{$K=8$}
	\label{fig.K8}
   \end{subfigure}
   
   \begin{subfigure}{0.99\textwidth}
   \includegraphics[width=0.98\textwidth]{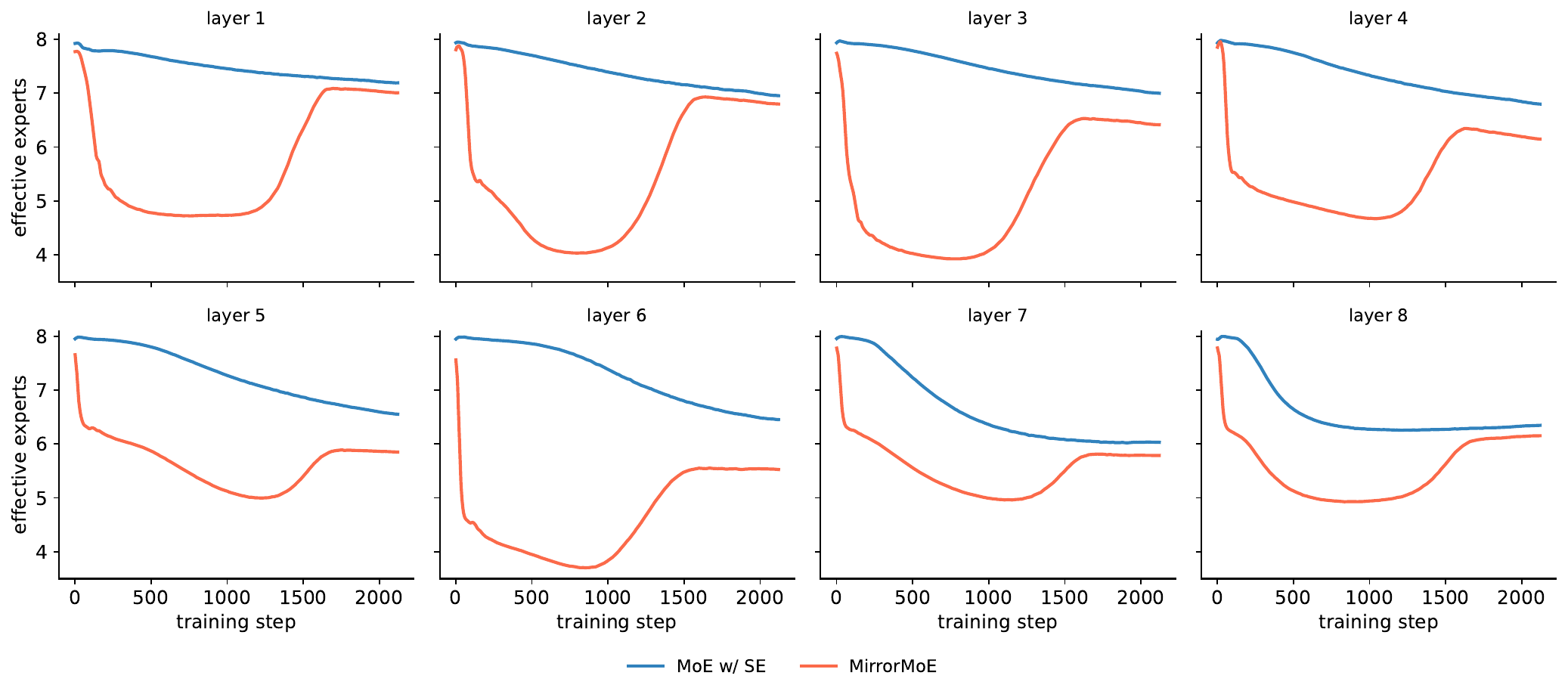}
    \caption{$K=10$}
	\label{fig.K10}
   \end{subfigure}
	\caption{Effective experts vs. iterations plot by layers.}
  \label{apdx.fig.effpaths}

 \end{figure}

\subsection{On the Load-Balancing Loss}\label{apdx.load-balancing}

\begin{table}[t]
\caption{Comparison of different load balancing loss for MirrorMoE}
\centering
\small
\begin{tabular}{lcc}
\toprule
		 & vanilla balancing loss &  our load balancing loss \eqref{eq.load-balance-loss} \\
\midrule
PPL        &  52.94   & 51.77    \\
runtime      & 7727.1 $\pm$  91.8  & 7844.4 $\pm$ 	85.6    \\
\bottomrule
\end{tabular}
\label{apdx.tab.balancing}
\end{table}

We use an MoE-100M-A38M to compare the standard load-balancing loss with our modified loss in \eqref{eq.load-balance-loss}. The two objectives differ in whether they balance the mirror routing probabilities $\{p_{l,e}\}_e$ or the original routing probabilities $\{\xi_{l,e}\}_e$. As shown in Tab.\ref{apdx.tab.balancing}, balancing $\{\xi_{l,e}\}_e$ reduces perplexity, at the cost of a slight increase in per-iteration training time. Since our primary objective is model performance, we adopt the modified load-balancing loss throughout our experiments.

\subsection{MirrorMoE performance degrades for MoEs without shared experts}\label{apdx.moe-ns}

\begin{table}[t]
\centering
\caption{MirrorMoE has less gains when pretraining MoEs without shared experts.}
\small
\setlength{\tabcolsep}{10pt}
\renewcommand{\arraystretch}{1.15}
\begin{tabular}{ccccc}
\toprule
\textbf{Model}
& {Experts}
& {Depth}
& \multicolumn{2}{c}{\textbf{PPL} ($\downarrow$)}
 \\
\cmidrule(lr){4-5}
Total-activated params
& $K$/$R$/$S$
& 
& {Vanilla}
& \cellcolor{blue!10} {\textbf{MirrorMoE}}
\\
\midrule
MoE-NS-97M-A36M
& 8/64/0
& 8S
& 52.01
& \cellcolor{blue!10} {51.23} 
\\

MoE-NS-480M-A115M
& 8/64/0
& 12S
& 28.24
& \cellcolor{blue!10}  {28.20}
\\

MoE-NS-614M-A95M
& 8/128/0
& 8S
& {28.42}
& \cellcolor{blue!10} 28.49
\\

MoE-NS-2.5B-A270M
& 8/128/0
& 16S
& 20.85
& \cellcolor{blue!10} 20.83
\\
\bottomrule
\end{tabular}
\label{tab.numerical.non-shared-moe}  
\end{table}

We further evaluate MirrorMoE on MoEs without shared experts. For ease of parallelization, we use $E\in\{64,128\}$ and $K=8$, with slightly adjusted model depths. These results are therefore not directly comparable to those in Tab.\ref{tab.numerical.test_ppl}. As shown in Tab.\ref{tab.numerical.non-shared-moe}, without shared experts, MirrorMoE with $\rho_0=0.4$ and a cosine schedule performs comparably to standard training, with no significant improvement. This result is consistent with our motivation: when shared experts are absent, the exploration dilemma is less pronounced, leaving less need to shrink the exploration space.

\subsection{Comparison with top-$p$ routing}

We additionally compare MirrorMoE with another adaptive routing approach, namely top-$p$ routing \citep{huang2024harder}. Top-$p$ routing activates the smallest set of experts whose cumulative probability exceeds $p$, where $p \in (0,1)$ is a threshold parameter.
Intuitively, it also adapts the number of activated experts over training. Early in training, when routing probabilities are relatively uniform, more experts are required to exceed the threshold $p$. As routing becomes more concentrated, fewer experts are sufficient.

However, this adaptivity makes a fair comparison challenging. In our experiments, we observe that the average number of activated experts per token decreases over the course of training. As a result, top-$p$ routing uses different amounts of activated computation at different training stages, making direct comparison with MirrorMoE and the baseline difficult, as both use a fixed number of activated experts. Tab.\ref{tab.top-p} reports the number of activated experts (including shared ones) averaged over layers and tokens for MoE-100M-A38M with $K=8$ and $S=2$.

In terms of training performance, we use entropy regularization with a coefficient of $10^{-4}$ for top-$p$ routing. As shown in Tab.\ref{tab.top-p}, even with $p=0.6$, which activates more experts on average, top-$p$ routing still underperforms MirrorMoE.

\begin{table}[t]
\caption{Comparison with top-\(p\) routing. Activated experts are averaged over layers and tokens for top-\(p\) routing.}
\centering
\small
\begin{tabular}{lccc}
\toprule
Method & PPL ($\downarrow$) & \shortstack{\# activated experts\\at first iteration} & \shortstack{\# activated experts\\at last iteration} \\
\midrule
Vanilla    & 54.78 & 8 & 8 \\
MirrorMoE  & \textbf{51.77}  & 8 & 8 \\
Top-\(p\), \(p=0.4\) & 55.64  & 18.65 & 6.73 \\
Top-\(p\), \(p=0.6\) & 54.09  & 29.46 & 10.79 \\
\bottomrule
\end{tabular}
\label{tab.top-p}
\end{table}

\section{Missing proofs}\label{sec.apdx.proofs}

\subsection{$\mathsf{ExpertVar}_l$ upper-bounds routing-induced changes}\label{apdx.derive-ub}

Consider $Z(\h_l)=\sum_{j\in {\cal A}_l(\h_l) } p_{l,j} (\mathbf{h}_l)$. 
To see that $\mathsf{ExpertVar}_l$ bounds the relative change induced by \textit{any} routing policies from above, let $\tilde{\h}_{l+1}(\bm{\alpha})=\sum_{e=1}^R \alpha_e \tilde{\h}_{l+1,e}$ where $\bm{\alpha} \in \mathbb{R}^R$ is any routing probability with $\bm{\alpha} \geq \mathbf{0}$, and $\mathbf{1}^\top \bm{\alpha} = 1$. One can easily see through convexity of $\| \cdot \|^2$ that

\begin{align*}
	\frac{\|\tilde{\h}_{l+1}(\boldsymbol{\alpha})-\h_{l+1}\|^2}{\|\h_{l+1}\|^2}
	&\leq\sum_{e=1}^R \alpha_e\frac{\|\tilde{\h}_{l+1, e}-\h_{l+1}\|^2}{\|\h_{l+1}\|^2} \\
	&\leq\sum_{e=1}^R \frac{\|\tilde{\h}_{l+1, e}-\h_{l+1}\|^2}{\|\h_{l+1}\|^2}
	= R \cdot \mathsf{ExpertVar}_l.
\end{align*}

\subsection{Proofs of Prop. \ref{prop.path-pruning}}

\begin{proof}
Recall that in MirrorMoE, one can write $p_e = (1 - \rho) \xi_e + \rho \zeta_e$. Using the detailed expression of $\xi_e$ and $\zeta_e$ in \eqref{eq.routing_prob}, one can rewrite $\zeta_e = \frac{1/\xi_e}{\sum_{r=1}^R 1/\xi_r}$. Given that $e_{\min}:=\arg\min_{e\in[R]}\xi_e$ with $\xi_{e_{\min}}=\epsilon$, we have that
\begin{align*}
	\sum_{e=1}^R \frac{1}{\xi_e} \leq \frac{R}{\xi_{e_{\min}}}.
\end{align*}

With $j:=\arg\min_{e\in{\cal A}^{\xi}}\xi_e$, it follows that
\begin{align*}
	\zeta_{e_{\min}}-\zeta_j 
	&= \frac{1/\xi_{e_{\min}}-1/\xi_j} {\sum_{r=1}^R 1/\xi_r} 
	\geq \frac{1}{R} \left( 1-\frac{\xi_{e_{\min}}}{\xi_j} \right) \\
	& = \frac{1}{R} \left( 1-\frac{\epsilon}{\xi_j} \right).
\end{align*}

Using the inequality above, we have that
\begin{align*}
	p_{e_{\min}}-p_j 
	&= -(1-\rho) \left(\xi_j-\xi_{e_{\min}}\right) + \rho \left(\zeta_{e_{\min}}-\zeta_j\right) \\
	&\geq -(1-\rho)\xi_j + \frac{\rho}{R} \left( 1-\frac{\epsilon}{\xi_j} \right).
\end{align*}

The right-hand side is positive whenever
\begin{align*}
	\rho > \frac{R\xi_j} {R\xi_j+1-\epsilon/\xi_j} = \frac{c}{c + 1 - \epsilon R/c} 
\end{align*}
where the last equation comes from $\xi_j = c/R$. As $\epsilon \rightarrow 0$, we get the condition for $\rho$ in the proposition. 

So far we have proved $p_{e_{\min}} \geq p_j$. To complete the proof, we will show that $e_{\min} \in {\cal A}^p$, i.e., $e_{\min}$ is indeed activated in MirrorMoE. 
Because $j:=\arg\min_{e\in{\cal A}^{\xi}}\xi_e$, every non-activated expert $e \notin {\cal A}^\xi$ satisfies $\xi_{e_{\min}} \leq \xi_e \leq \xi_j$.
Hence, there exists some $\lambda_e\in[0,1]$ such that
\begin{align*}	
	\xi_e = \lambda_e\xi_{e_{\min}} + (1-\lambda_e)\xi_j.
\end{align*}

Since $f(x)= 1/x$ is convex on $\mathbb{R}_{++}$, we have that
\begin{align*}
	\frac{1}{\xi_e} \leq \frac{\lambda_e}{\xi_{e_{\min}}} + \frac{1-\lambda_e}{\xi_j}.
\end{align*}

Denote $C = \left(\sum_{e=1}^R\frac{1}{\xi_e}\right)^{-1}$. It then follows that
\begin{align*}
		p_e 
	& = (1-\rho)\xi_e +\rho \frac{C}{\xi_e} \\ 
	& \leq \lambda_e \left( (1-\rho)\xi_{e_{\min}} + \frac{\rho C}{\xi_{e_{\min}}} \right) +
(1-\lambda_e) \left( (1-\rho)\xi_j + \frac{\rho C}{\xi_j} \right) \\
	& = \lambda_e p_{e_{\min}} + (1-\lambda_e)p_j \\
	& \leq p_{e_{\min}}, \forall e \notin {\cal A}^\xi
\end{align*}
where the final inequality follows from $p_{e_{\min}}>p_j$. 
Combining $p_{e_{\min}} \geq p_j$ and $p_{e_{\min}} \geq p_e, \forall e \notin {\cal A}^\xi$, it can be seen that ${e_{\min}} $ has to be activated in MirrorMoE, i.e., ${e_{\min}} \in {\cal A}^p$.

Finally, since $e_{\min} \in {\cal A}^p$, according to the aggregation in \eqref{eq.mirror-agg}, $\xi_{e_{\min}} = \epsilon \rightarrow 0$ ensures that $e_{\min}$ is functionally suppressed. The proof is thus completed.
\end{proof}

\subsection{Proof of Prop. \ref{prop.no-path-pruning}}
\begin{proof}
Let $ C = \left( \sum_{r=1}^R \frac{1}{\xi_r} \right)^{-1}$. Recall from \eqref{eq.routing_prob}, $p_e = (1-\rho)\xi_e + \rho\zeta_e $ and $\zeta_e=\frac{C}{\xi_e}$.

We first show that $C \leq\frac{1}{R^2}$ via Cauchy-Schwarz inequality. To see this, we have 
\begin{align*}
	\left(\sum_{r=1}^R\xi_r\right)
	\left(\sum_{r=1}^R\frac{1}{\xi_r}\right) \geq R^2.
\end{align*}
It is easy to see $C \leq\frac{1}{R^2}$ once noticing $\sum_r\xi_r=1$.

Now, consider any two experts $i$ and $j$ satisfying $\xi_i>\xi_j$. We have
\begin{align*}
	  p_i - p_j 
	& = (1-\rho)(\xi_i-\xi_j) + \rho C \left( \frac{1}{\xi_i}-\frac{1}{\xi_j} \right) \\
	& = (\xi_i-\xi_j) \left[ (1-\rho) - \rho \frac{C}{\xi_i\xi_j} \right].
\end{align*}

Since $\xi_i\xi_j \geq \xi_{\min}^2$ and $\xi_{\max}\geq 1/R$, it can be seen that
\begin{align*}
	\frac{C}{\xi_i\xi_j} \leq \frac{1}{R^2 \xi_i\xi_j} 
	\leq \frac{1}{R^2\xi_{\min}^2} 
	\leq \left( \frac{\xi_{\max}}{\xi_{\min}} \right)^2 
	< \kappa^2 = \frac{1-\rho}{\rho}.
\end{align*}
Combining the above two equations, it is straightforward to see  $\xi_i>\xi_j \implies p_i > p_j $ for all $i$ and $j$. Thus, it can be seen that the ordering of experts are same if ranked with $\{ \xi_e\}$ and  $\{ p_e\}$. This indicates the top-$(K-S)$ aggregation will choose the same set of experts, i.e., ${\cal A}^\xi = {\cal A}^p$.

It remains to show that none of the selected experts has a vanishing score. For any selected expert $e\in{\cal A}^p$,
\begin{align*}
	\bar{\xi}_e = \frac{\xi_e} {\sum_{r\in{\cal A}^p}\xi_r} \geq \frac{\xi_{\min}}{(K-S)\xi_{\max}} > \frac{1}{(K - S)\kappa}.
\end{align*}
Thus, every selected expert has a non-vanishing aggregation weight, which rules out functional path pruning.
\end{proof}

\end{document}